\documentclass{article}

\usepackage[final,nonatbib]{nips_2016}

\usepackage[utf8]{inputenc} 
\usepackage[T1]{fontenc}    
\usepackage{hyperref}       
\usepackage{url}            
\usepackage{booktabs}       
\usepackage{amsfonts}       
\usepackage{nicefrac}       
\usepackage{microtype}      

\usepackage[numbers,square,comma,sort&compress]{natbib}
\usepackage{times}
\usepackage{calc}
\usepackage{mathtools}
\usepackage{amsmath, amssymb, amsthm}
\usepackage{color}
\usepackage{graphicx} 
\usepackage{subfigure}
\usepackage{algorithm}
\usepackage{algorithmic}
\usepackage{booktabs}
\usepackage{enumitem}
\usepackage{bbm}
\usepackage{multirow}
\usepackage{pifont}
\usepackage{kky}
\usepackage{herbertDefns2}
\usepackage{stmaryrd}

\newcommand{\toworkon}[1]{}

\title{The Multi-fidelity Multi-armed Bandit}

\newcommand{\instcmu}{$\,^\natural$}
\newcommand{\instrice}{$\,^\diamondsuit$}
\newcommand{\authspace}{$\;\;$}

\author{
Kirthevasan Kandasamy\instcmu, \authspace
Gautam Dasarathy\instrice, \authspace
Jeff Schneider\instcmu, \authspace
Barnab\'as P\'oczos\instcmu 
\\
  \instcmu $\,$Carnegie Mellon University,  \authspace
  \instrice $\,$Rice University \\
 {\small\texttt{\{kandasamy, gautamd, schneide, bapoczos\}@cs.cmu.edu}}
   \\
}

\begin{document}
\pdfoutput=1

\maketitle

\newcommand{\insertAlgoMFUCB}{
\begin{algorithm}
\begin{itemize}
\item for $t = 1, 2,\dots$
  \begin{enumerate}
  \item Choose $\It \in \argmax_{k\in\allArms} \,\BBkt$. 
    \hspace{0.2in} (See equation~\eqref{eqn:BktDefn}.)
  \item $\mt = \min_m \;\{\,m \;|\; 
          \psiinv(\rho\log t/\TmIttmo) \geq \gammam
        \;\;\vee\;\; m = M \}$ 
    \hspace{0.2in} (See equation~\eqref{eqn:gammamdefns}.)
  \item Play $X\sim\PmtIt$.
  \end{enumerate}
\end{itemize}
\caption{\hspace{0.05in}\mfucb \label{alg:mfucb}}
  \vspace{-0.05in}
\end{algorithm} }

\newcommand{\insertEETmknTable}{
\newcommand{\diagcolour}{black}
\begin{table}
\centering
\begin{tabular}{|c|c|c|p{0.5cm}cp{0.5cm}|c|c|}
  \hline
 & $\Kcalmm{1}$ & $\Kcalmm{2}$ & & $\Kcalm$ & & $\KcalM$ & $\optArms$ \\
\hline
$\EE[\Tmmkktt{1}{k}{n}]$ & $\textcolor{\diagcolour}{\frac{\logn}{\psiDeltaonekb}}$ &
$\frac{\logn}{\psigammaone}$  & 
\dots & $\frac{\logn}{\psigammaone}$ & \dots &  
 $\frac{\logn}{\psigammaone}$ & $\frac{\logn}{\psigammaone}$  \\
\hline
$\EE[\Tmmkktt{2}{k}{n}]$ & \multirow{7}{*}{$\bigO(1)$} & 
  $\textcolor{\diagcolour}{\frac{\logn}{\psiDeltatwokb}}$&
\dots & $\frac{\logn}{\psigammatwo}$ & \dots &  
 $\frac{\logn}{\psigammatwo}$ & $\frac{\logn}{\psigammatwo}$  \\
\cline{1-1} \cline{3-8}
$\vdots$ & & & & & & & \\
$\EE[\Tmkn]$ & & \multirow{4}{*}{$\bigO(1)$} & \dots &
  $\textcolor{\diagcolour}{\frac{\logn}{\psiDeltamkb}}$&
\dots & $\frac{\logn}{\psigammam}$ &  $\frac{\logn}{\psigammam}$ \\
$\vdots$ & & & & & & & \\
\cline{1-1} \cline{4-8}
$\EE[\TMkn]$ & & & & $\bigO(1)$ & & 
  $\textcolor{\diagcolour}{\frac{\logn}{\psiDeltaMkb}}$ & $\bigOmega(n)$ \\
\hline
\end{tabular}
\vspace{-0.00in}
\caption{
Bounds on the expected number of plays for each $k\in\Kcalm$ (columns) at each
fidelity (rows) after $n$ time steps (i.e. $n$ plays at any fidelity)  in \mfucb.
\label{tb:upperbound}
\vspace{-0.20in}
}
\end{table}
}

\newcommand{\insertFigSets}{
\begin{figure}
\centering
  \begin{minipage}[c]{2.4in}
    \includegraphics[width=2.3in]{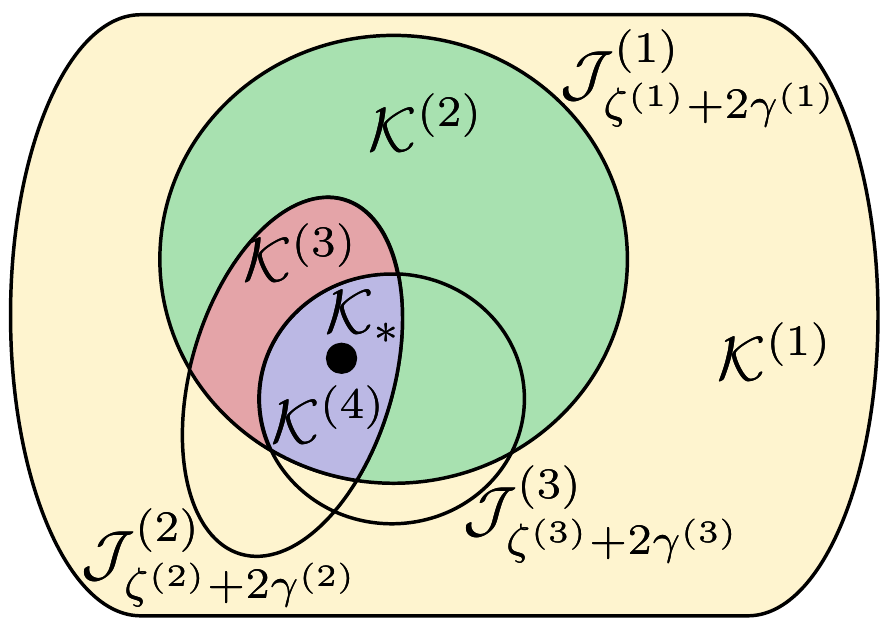}
  \end{minipage} \hspace{0.1in}
  \begin{minipage}[l]{2.7in}
  \vspace{-0.1in}
    \caption{
  Illustration of the partition $\Kcalm$'s for a $M=4$ fidelity problem. The sets 
  $\Jcalmzg$
are indicated next to their boundaries. $\Kcalone,\Kcaltwo,\Kcalmm{3},\Kcalmm{4}$ are
shown in yellow, green, red and purple respectively. The optimal arms $\optArms$ 
are shown as a black circle. 
    } 
\label{fig:Kcalms}
  \end{minipage}
  \vspace{-0.00in}
\end{figure}
}

\newcommand{\imarrwthree}{1.87in}
\newcommand{\imhspthree}{-0.15in}
\newcommand{\imcorners}{-0.25in}

\newcommand{\insertSimOneFigure}{
\begin{figure*}
\centering
\hspace{\imcorners}
\subfigure{
  \includegraphics[width=\imarrwthree]{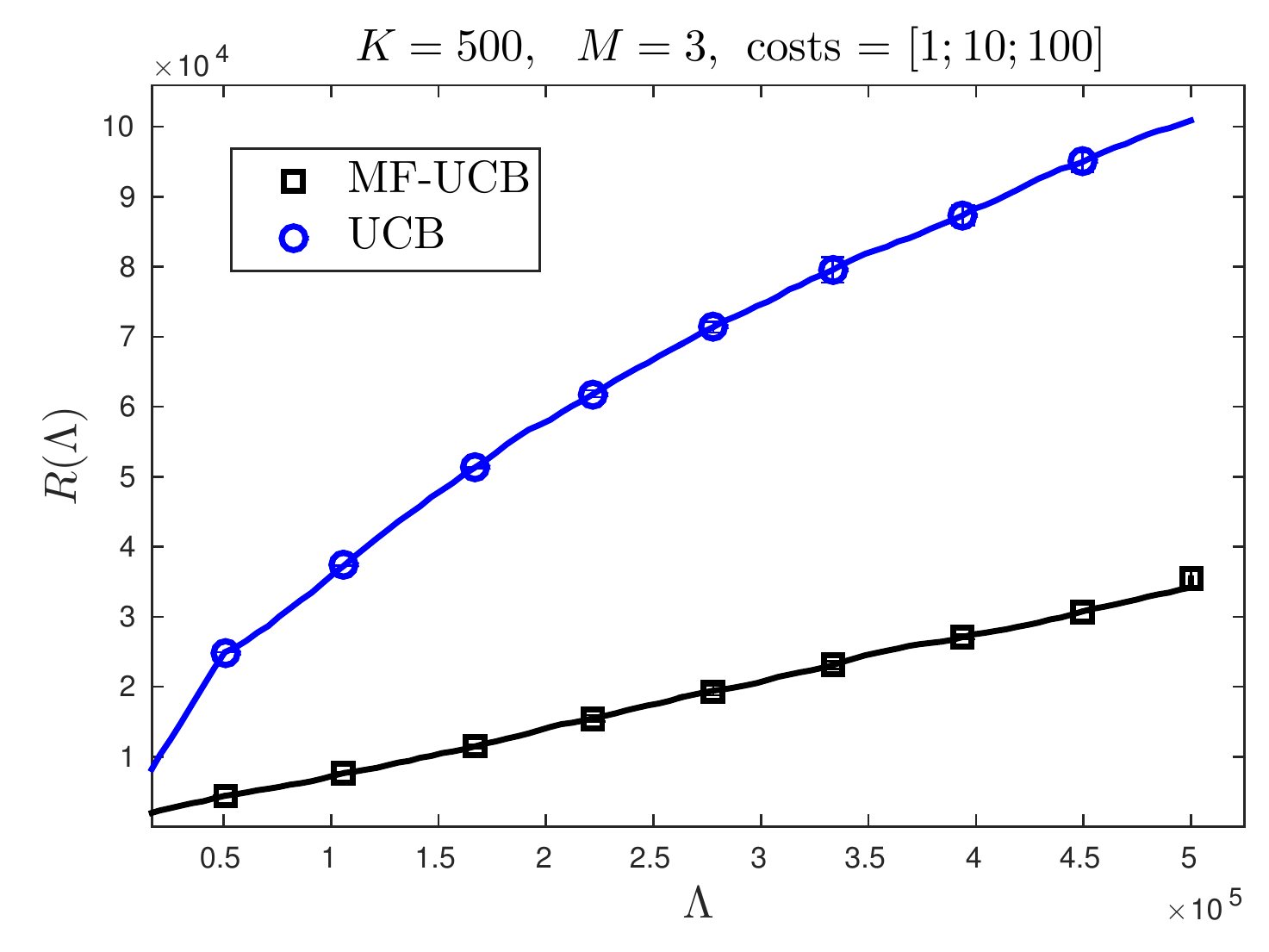} \hspace{\imhspthree}
  \label{fig:G500}
} 
\subfigure{
  \includegraphics[width=\imarrwthree]{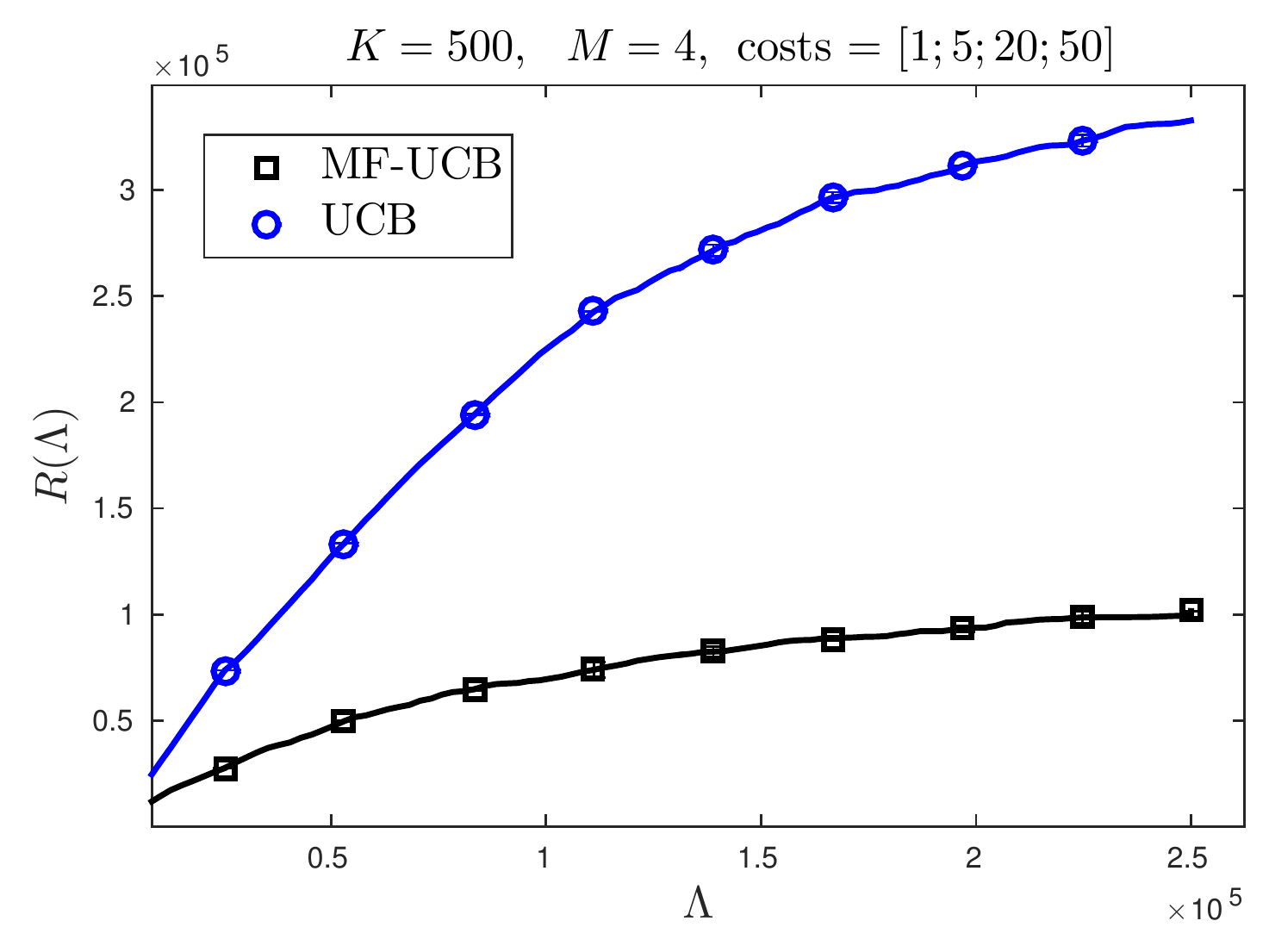} \hspace{\imhspthree}
  \label{fig:G400}
}
\subfigure{
  \includegraphics[width=\imarrwthree]{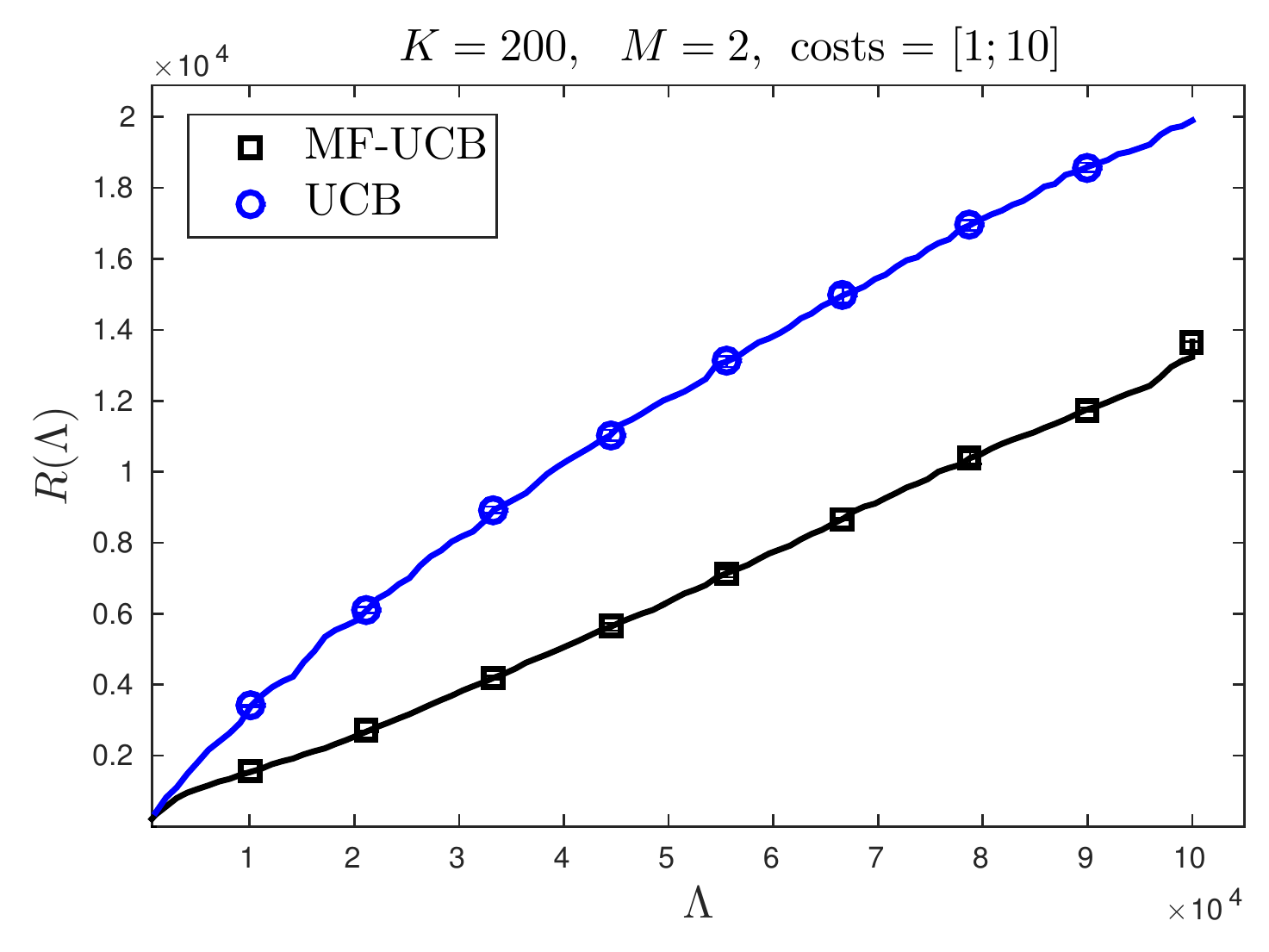} \hspace{\imcorners}
  \label{fig:B200}
}
\\[-0.15in]
\hspace{\imcorners}
\subfigure{
  \includegraphics[width=\imarrwthree]{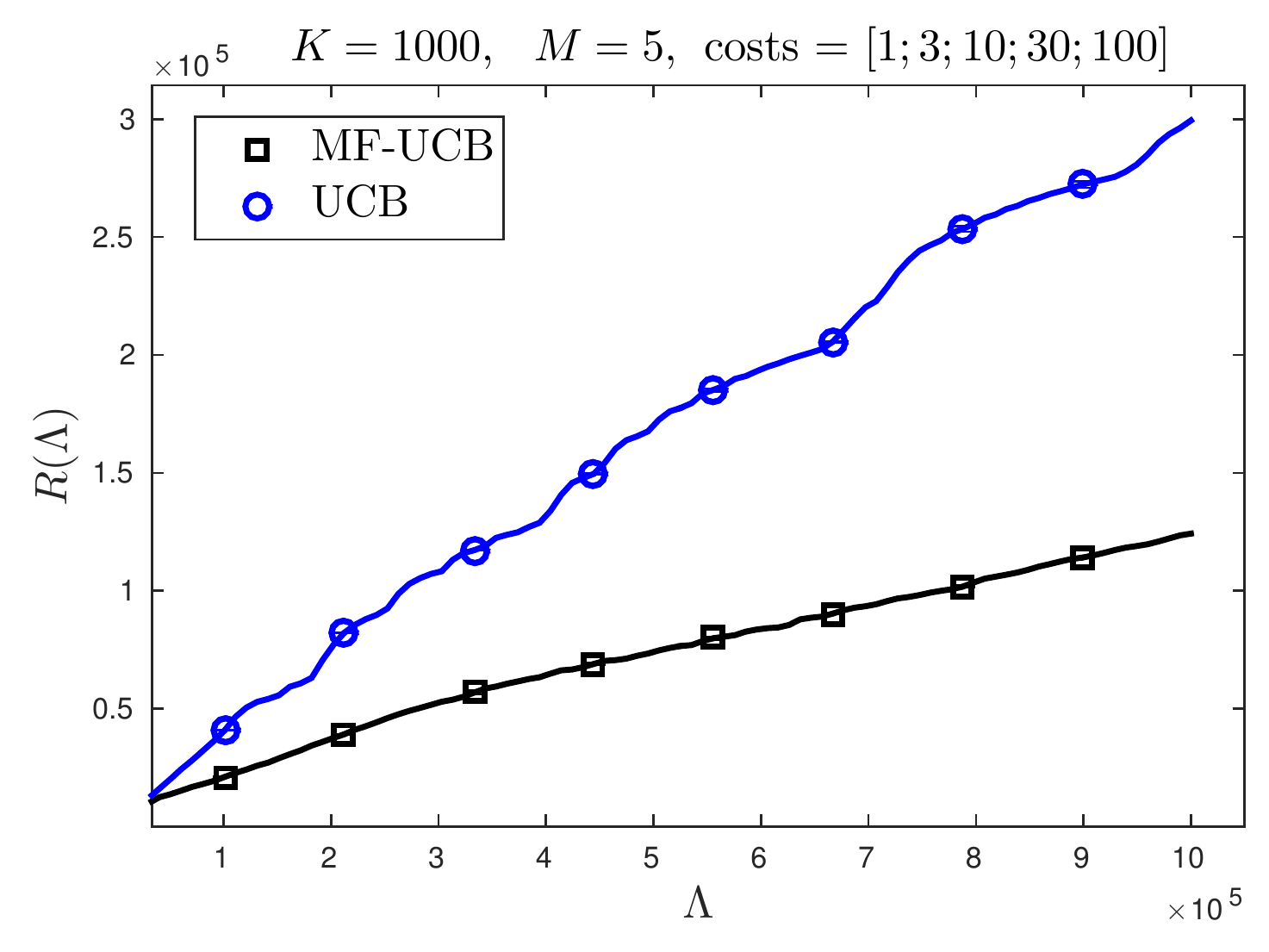} \hspace{\imhspthree}
  \label{fig:B1000}
}
\subfigure{
  \includegraphics[width=\imarrwthree]{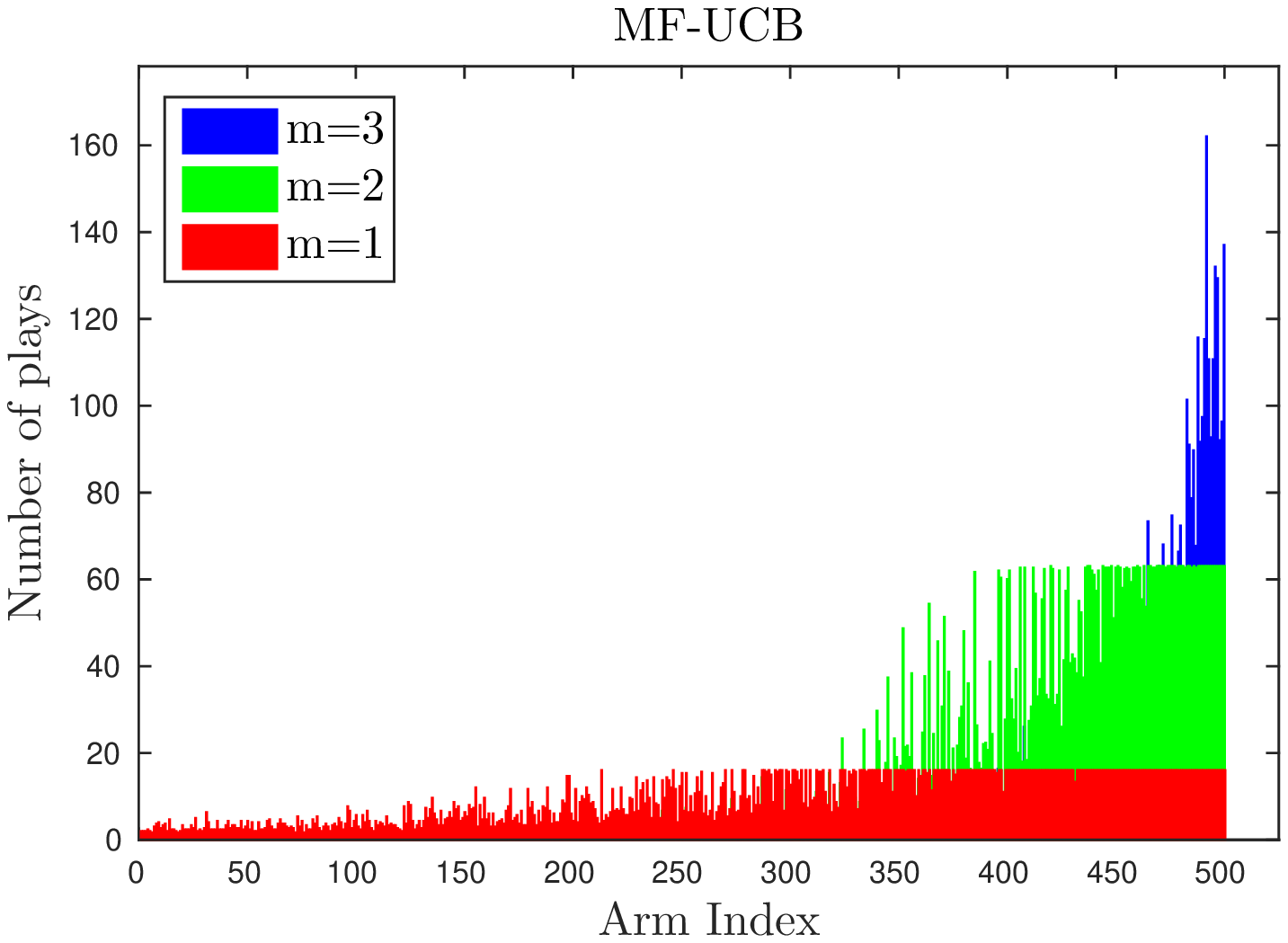} \hspace{\imhspthree}
  \label{fig:G500mfucb}
}
\subfigure{
  \includegraphics[width=\imarrwthree]{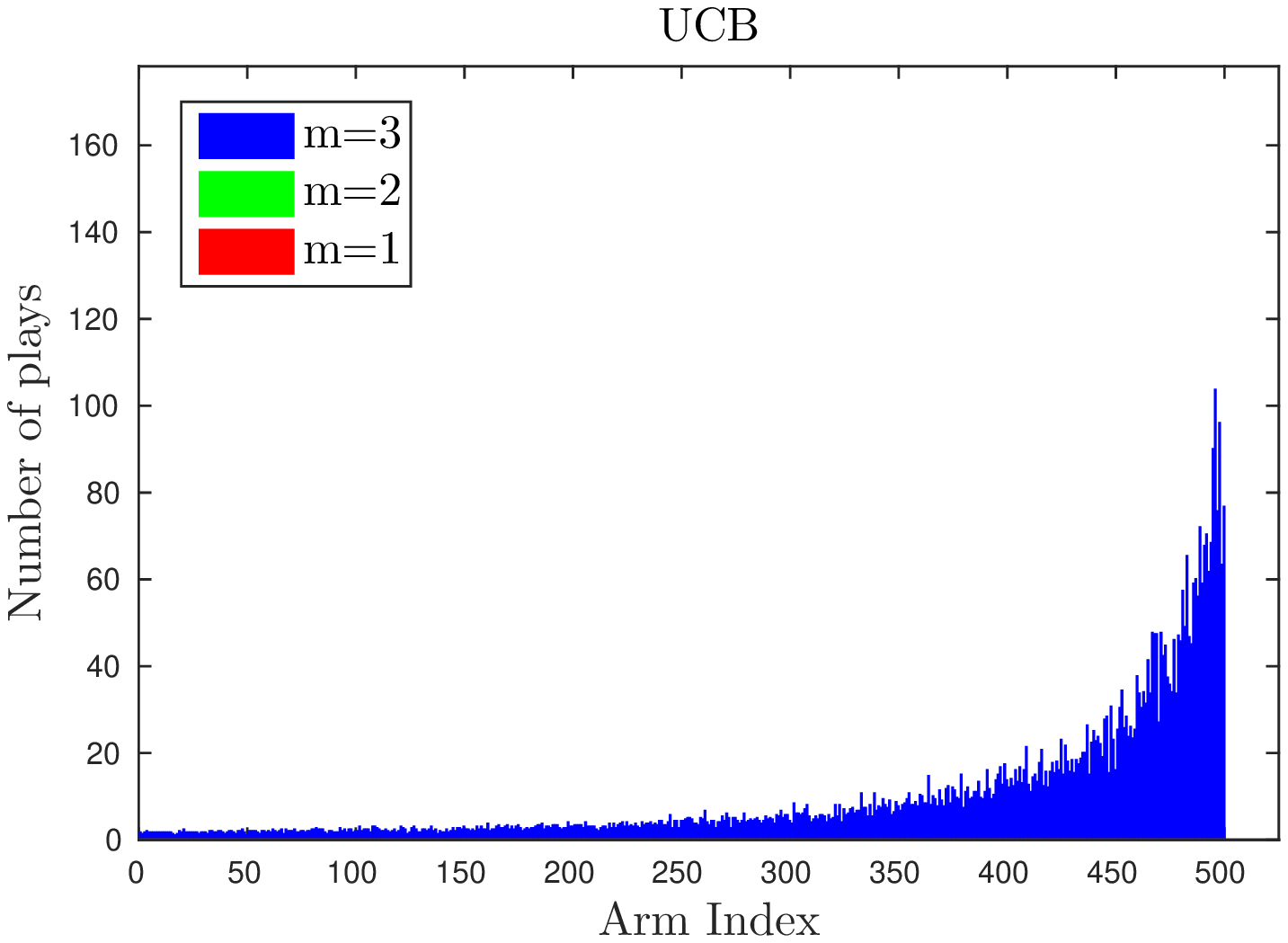} \hspace{\imcorners}
  \label{fig:G500ucb}
}
\\ \vspace{-0.1in}
\caption[]{\small
\label{fig:simOne}
Simulations results on the synthetic problems. 
The first four figures compares \ucbs against \mfucbs on four synthetic problems.
The title states $K,M$ and the costs $\costone,\dots,\costM$.
The first two used Gaussian rewards and the last two used Bernoulli rewards.
The last two figures show the number 
of plays by \ucbs and \mfucbs on a $K=500, M=3$ problem with Gaussian observations
(corresponding to the first figure).
\vspace{-0.1in}
}
\end{figure*}
}

\newcommand{\imhsptwo}{0.1in}
\newcommand{\imarrwtwo}{2.3in}
\newcommand{\insertArmsFigure}{
\begin{figure*}
\centering
\hspace{\imcorners}
\subfigure{
  \includegraphics[width=\imarrwtwo]{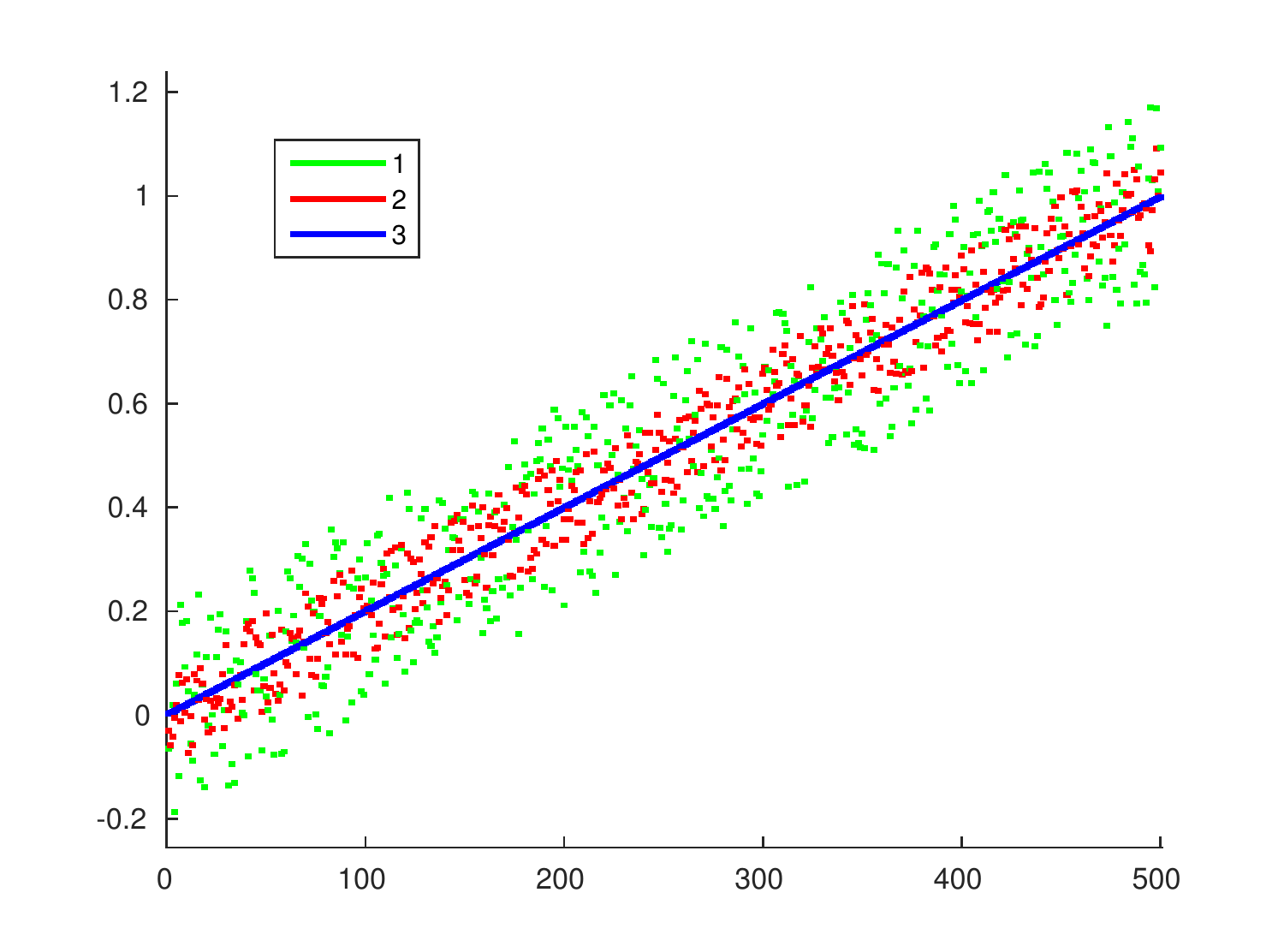} \hspace{\imhsptwo}
} 
\subfigure{
  \includegraphics[width=\imarrwtwo]{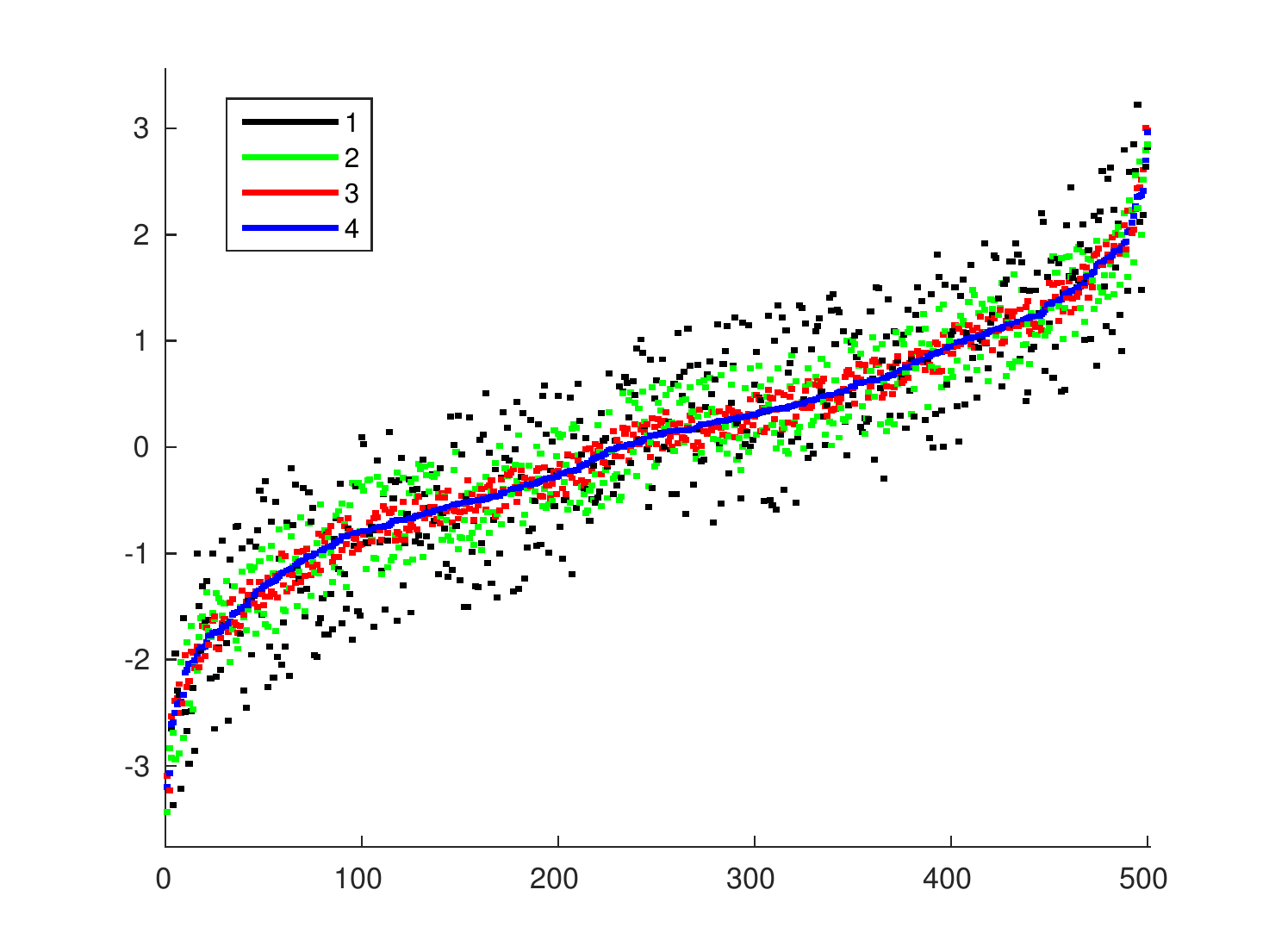} 
}
\\[-0.15in]
\hspace{\imcorners}
\subfigure{
  \includegraphics[width=\imarrwtwo]{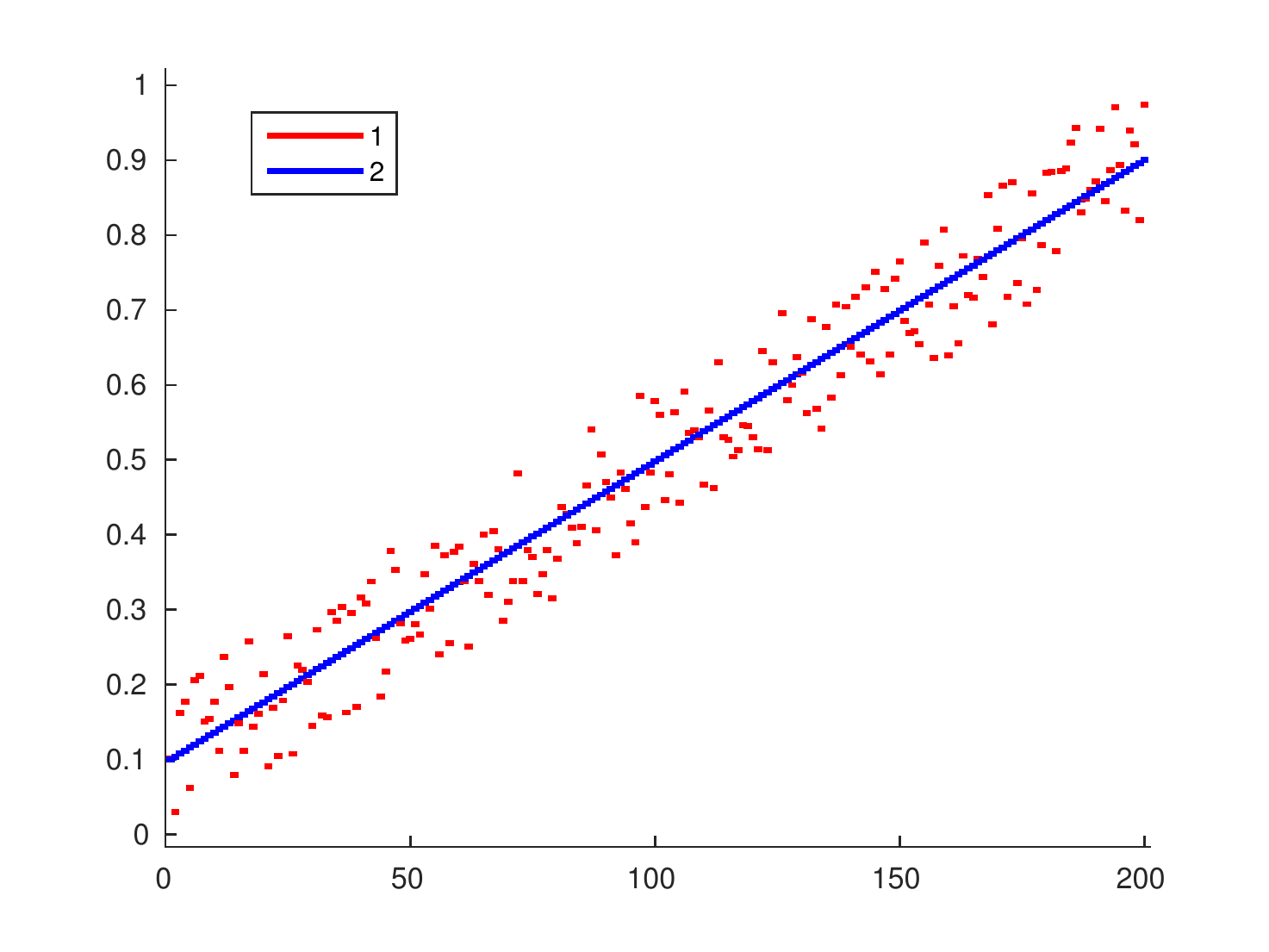} \hspace{\imhsptwo}
}
\subfigure{
  \includegraphics[width=\imarrwtwo]{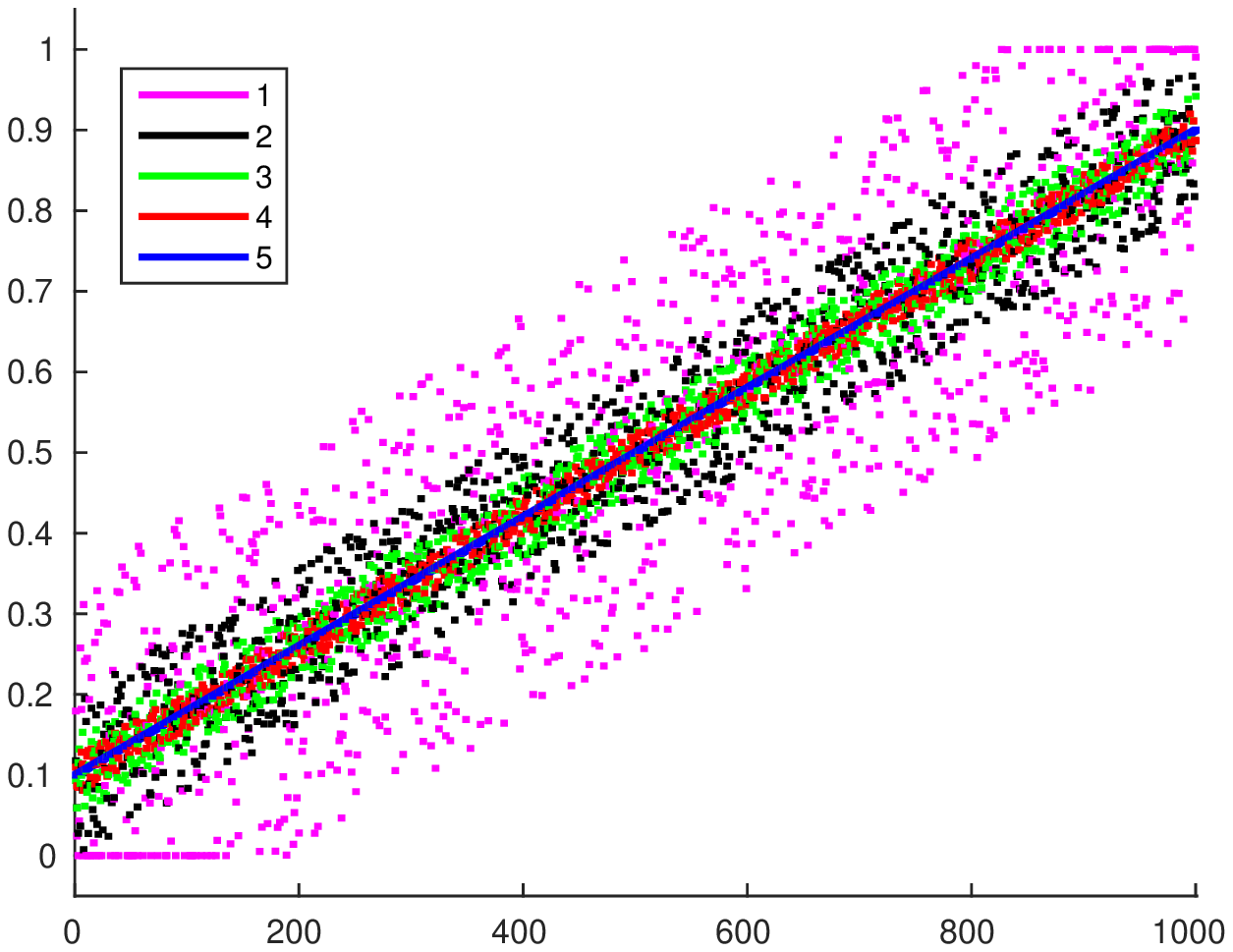} 
}
\\ \vspace{-0.1in}
\caption[]{\small
\label{fig:simArms}
An illustration of the means of the arms used the simulation probelms.
The top row are the Gaussian rewards with $(K,M)$ equal to $(500,3)$, $(500,4)$
while the second row are the Bernoulli rewards with $(200,2)$, $(1000,5)$
respectively.
\vspace{-0.1in}
}
\end{figure*}
}

\begin{abstract}
We study a variant of the classical stochastic $K$-armed bandit where 
observing the outcome of each arm is expensive, but cheap approximations to this
outcome are available.
For example, in online advertising the performance of an ad can be approximated
by displaying it for shorter time periods or to narrower audiences.
We formalise this task as a \emph{multi-fidelity} bandit, where, at each time
step, the forecaster may choose to play an arm at any one of $M$ fidelities.
The highest fidelity (desired outcome) expends cost $\costM$. 
The $m$\ssth fidelity (an approximation) expends $\costm < \costM$ and returns
a biased estimate of the highest fidelity.
We develop \mfucb, a novel upper confidence bound procedure for this setting and 
prove that it naturally adapts to the sequence of available
approximations and costs thus attaining better regret than naive strategies which ignore
the approximations.
For instance, in the above online advertising example,
\mfucbs would use the lower fidelities to quickly eliminate
suboptimal ads and reserve the larger expensive experiments on a small set of
promising candidates. 
We complement this result with a lower bound and show that \mfucbs is nearly optimal
under certain conditions.
\end{abstract}

\section{Introduction}
\label{sec:intro}
Since the seminal work of~\citet{robbins52seqDesign}, the multi-armed bandit
has become an attractive framework for 
studying exploration-exploitation trade-offs inherent to tasks arising in
online advertising, finance and other fields. 
In the most basic form of the $K$-armed
bandit~\citep{thompson33sampling, lai85bandits}, 
we have a set $\allArms =\{1,\dots,K\}$ of $K$ arms (e.g. 
$K$ ads in online advertising). 
At each time step $t=1,2,\dots$, an arm is played and a corresponding reward is realised. 
The goal is to design a strategy of plays 
that minimises the \emph{regret} after $n$ plays. The regret is the comparison,
in expectation, of
the realised reward against an oracle that always plays the best arm. 
The well known Upper Confidence Bound (\ucb) algorithm~\citep{auer03ucb}, achieves 
regret $\bigO(K\logn)$
after $n$ plays  (ignoring mean rewards) and is minimax optimal~\citep{lai85bandits}.
%

In this paper, we propose a new take on this important problem. In many practical
scenarios of interest, one can associate a cost to playing each arm. Furthermore, in many
of these scenarios, one might have access to cheaper
approximations to the outcome of the arms. 
For instance, in online advertising the goal is to maximise the cumulative number of
clicks over a given time period. Conventionally, an arm pull maybe thought of as the
display of an ad for a specific time, say one hour.
However, we may approximate its hourly performance by displaying the ad for shorter
periods. This estimate 
is \emph{biased} (and possibly noisy), as
displaying an ad for longer intervals changes user behaviour. 
It can nonetheless be useful in gauging the long run click through rate. 
We can also obtain biased estimates of an ad by displaying it only to certain geographic
regions or age groups.
Similarly one might consider \emph{algorithm selection} for machine learning
problems~\cite{baram2004online}, where the goal is to be competitive with the
best among a set of learning algorithms for a task.
Here, one might obtain cheaper approximate estimates of the performance of algorithm
by cheaper versions using less data or computation.
In this paper, we will refer to such approximations as fidelities.
Consider a $2$-fidelity problem where the cost at the low fidelity is
$\costone$ and the cost at the high fidelity is $\costtwo$.
We will present a cost weighted notion of regret for this setting for a
strategy that expends a capital of $\COST$ units. 
A classical $K$-armed bandit strategy such as \ucb, which only uses the highest
fidelity, can obtain at best 
$\bigO(\costtwo K\log(\COST/\costtwo))$ regret~\citep{lai85bandits}.
In contrast, this paper will present 
multi-fidelity strategies that achieve $\bigO\left((\costone K +
\costtwo\nKg)\log(\COST/\costtwo)\right)$ regret.
Here $\Kcalg$ is a (typically) small subset of arms with high 
expected reward that can 
be identified using plays at the (cheaper) low fidelity. 
When $\nKg < K$ and $\costone<\costtwo$, such a strategy will outperform the more
standard \ucbs algorithms. 
Intuitively, this is achieved by using the lower fidelities to eliminate several of ``bad'' arms 
and reserving expensive higher fidelity plays  for a small subset of the
most promising arms.
%
We formalise the above intuitions in the sequel. Our main contributions are,
\vspace{-0.05in}
\begin{enumerate}[leftmargin=0.25in]
\item A novel {\bf formalism} for studying bandit tasks
when one has access to multiple fidelities for each arm, with each successive
fidelity  providing a better approximation to the most expensive one.
\vspace{-0.05in}
\item A new {\bf algorithm} that we call Multi-Fidelity Upper Confidence Bound 
(\mfucb) that adapts the classical Upper Confidence Bound (UCB) strategies 
to our multi-fidelity setting. 
Empirically, we demonstrate that our algorithm outperforms naive \ucbs on 
simulations. \hspace{-0.4in}
\vspace{-0.05in}
\item A {\bf theoretical characterisation} of the performance of \mfucbs that 
shows that the algorithm (a) uses the lower fidelities to explore all arms and 
eliminates arms with low expected reward, and (b) reserves the higher fidelity plays 
for arms with rewards close to the optimal value. We derive a lower bound on 
the regret and 
demonstrate that \mfucbs is near-optimal on this problem.
\vspace{-0.05in}
\end{enumerate}

\vspace{-0.1in}
\subsection*{Related Work}
\vspace{-0.10in}
The $K$-armed bandit has been studied extensively in
the past~\citep{robbins52seqDesign,agrawal95bandit,lai85bandits}.
There has been a flurry of work on upper confidence bound (\ucb)
methods~\citep{audibert09multiarmedbandit,auer03ucb}, which adopt the optimism in the
face of uncertainty principle for bandits.
For readers unfamiliar with
\ucbs methods, we recommend Chapter 2 of~\citet{bubeck12regret}.
Our work in this paper builds on \ucbs ideas, but the multi-fidelity framework poses
significantly new algorithmic and theoretical challenges. 


There has been some interest in multi-fidelity methods for optimisation in many 
applied domains of research~\citep{huang06mfKriging,rajnarayan08mfAerodynamic}.
However, these works do not formalise or analyse notions of 
\emph{regret} in the multi-fidelity setting. 
Multi-fidelity methods are used in the robotics community for
reinforcement learning tasks by modeling each fidelity as a Markov decision 
process~\citep{cutler14mfsim}.
\citet{zhang15weakAndStrong} study active learning with a cheap weak labeler and an
expensive strong labeler. The objective of these papers however is not to handle the
exploration-exploitation trade-off inherent to the bandit setting. 
A line of work on budgeted multi-armed bandits~\citep{xia15thompson,thanh14budget}
study a variant of the $K$-armed bandit where each arm has a random reward and cost
and the goal is to play the arm with the highest reward/cost ratio as much as
possible. This is different from our setting where each arm has multiple fidelities
which serve as an approximation.
Recently, in~\citet{kandasamy16mfbo} we extended ideas in this work to analyse
multi-fidelity bandits with Gaussian process payoffs.

\section{The Stochastic $K$-armed Multi-fidelity Bandit}
\label{sec:prelims}

In the classical $K$-armed bandit, each arm $k\in\allArms = \{1,\dots,K\}$
is associated with a real valued distribution $\Pk$ with mean $\muk$. 
Let $\optArms = \argmax_{k\in\allArms} \muk$ be the set of optimal arms, 
$\kopt\in\optArms$ be an optimal arm and $\muMopt = \mu_{\kopt}$ denote the optimal 
mean value.
A bandit strategy would \emph{play} an arm $\It\in \allArms$ at each
time step $t$ and observe a sample from $\Pkk{\It}$.
Its goal is to maximise the sum of 
expected rewards  after $n$ time steps 
$\sum_{t=1}^n\muIt$, or equivalently minimise the cumulative pseudo-regret
$\sum_{t=1}^n\muMopt - \muIt$ for \emph{all values} of $n$.
In other words, the objective is to be competitive, 
in expectation, against an oracle that plays an optimal arm all the time.

In this work we differ from the usual bandit setting in the following aspect.
For each arm $k$, we have access to $M-1$ successively approximate distributions
$\Ponek,\Ptwok,\dots,\Pmmkk{M-1}{k}$ to the desired distribution $\PMk = \Pk$.
We will refer to these approximations as fidelities. Clearly, these approximations are
meaningful only if they give us some information about $\PMk$. 
In what follows, we will assume that the $m$\ssth 
fidelity mean of an arm is within $\zetam$, a \emph{known quantity}, of its highest 
fidelity mean, where $\zetam$, decreasing with $m$,  characterise the successive
approximations. That is,
$|\muMk-\mumk| \leq \zetam$ for all $k\in\allArms$ and $m=1,\dots,M$,
where 
$\zetaone > \zetatwo > \dots > \zetaM=0$ and
the $\zetam$'s are known.
It is possible for the lower fidelities to be misleading under this
assumption: there could exist an arm $k$ with $\muMk <
\muMopt=\mummkk{M}{\kopt}$ but with
$\mummkk{m}{k}>\muMopt$ and/or  $\mummkk{m}{k}>\mumkk{\kopt}$ for any
$m<M$.
In other words, we wish to explicitly account for the \emph{biases}
introduced by the lower fidelities, and not treat them as just a higher variance
observation of an expensive experiment.
This problem of course becomes interesting only when lower 
fidelities are more attractive than higher fidelities in terms of some notion
of cost. 
Towards this end, we will assign a cost $\costm$ 
(such as advertising time, money etc.) to playing an arm at fidelity $m$ where
$\costone < \costtwo \dots < \costM$.

\textbf{Notation: }
$\Tmkt$ denotes the number of plays at arm $k$, at
fidelity $m$ until $t$ time steps. $\Tgmkt$ is the number of plays at fidelities
greater than $m$. $\Tmt=\sum_{k\in\allArms}\Tmkt$ is the number of fidelity 
$m$ plays at all arms until time $t$.
$\Xbarmks$ denotes the mean of $s$ samples drawn from $\Pmk$. 
Denote $\Deltamk = \muMopt - \mumk - \zetam$.
When $s$ refers to the number of plays of an arm,
we will take $1/s = \infty$ if $s=0$. 
$\complement{A}$ denotes the complement of a set $A\subset\Kcal$.
While discussing the intuitions in our proofs and theorems we will use
$\asymp,\lesssim,\gtrsim$ to denote equality and inequalities ignoring constants.

\textbf{Regret in the multi-fidelity setting:}
A strategy for a multi-fidelity bandit problem, at time $t$, 
produces an arm-fidelity pair $(\It, \mt)$, where $\It\in \Kcal$ and 
$\mt \in \{1,\dots,M\}$, 
and observes a sample $X_t$ drawn (independently of everything else) from the distribution $\theta_{\It}^{(\mt)}$.  The choice of 
$(\It,\mt)$ could depend on previous arm-observation-fidelity tuples 
$\{(I_i,X_i,m_i)\}_{i=1}^{t-1}$.
The multi-fidelity setting calls for a new notion of regret.
For any strategy $\Acal$ that expends $\COST$ units of the resource, we will define the
pseudo-regret $R(\COST, \Acal)$  as follows.
Let $\qt$ denote the \emph{instantaneous pseudo-reward} at time $t$ and
$\rt=\muMopt - \qt$ denote the instantaneous pseudo-regret.
We will discuss choices for $\qt$ shortly.
Any notion of regret in the multi-fidelity setting needs to 
account for this instantaneous regret along with the cost of the fidelity at which we
played at time $t$, i.e. $\costmt$. 
Moreover, we should receive no reward (maximum regret) for any unused capital.
These observations lead to the following definition,
\begin{align*}
R(\COST,\Acal) \;=\;
\COST\muMopt - \sumtN\costmt\qt \;=\;
  \underbrace{\left(\COST - \sumtN\costmt\right)\muMopt}_{\rtilde(\COST,\Acal)} 
  \;\;+\;\;
  \underbrace{\sumtN \costmt \rt}_{\Rtilde(\COST,\Acal)}.
\numberthis \label{eqn:regretDefn} \\[-0.20in]
\end{align*}
Above, $N$ is the (random) number of plays within
capital $\COST$ by $\Acal$, 
i.e. the largest $n$ such that $\sum_{t=1}^n\costmt \leq \COST$. 
To motivate our choice of $\qt$ we consider an online advertising example where
$\costm$ is the advertising time at fidelity
$m$ and $\mumk$ is the expected number of clicks per unit time. While we
observe from $\PmtIt$ at time $t$, we wish to reward the strategy according to
its highest fidelity distribution $\PMIt$. 
Therefore regardless of which fidelity we play we set $\qt = \muMIt$. 
Here, we are competing
against an oracle which plays an optimal arm at any fidelity all the time.
%
Note that we might have chosen $\qt$ to be $\mumtIt$. However, this does not reflect the motivating applications for the multi-fidelity setting that we consider. For instance, a clickbait ad might receive a high number of clicks in the short run, but its long
term performance might be poor. Furthermore, for such a choice, we may as well ignore the rich structure inherent to the multi-fidelity setting and simply play the arm $\argmax_{m,k}\mumk$ at each time. There are of course other choices for $\qt$ that result in very different notions of regret; we discuss this briefly at the end of Section~\ref{sec:conclusion}.

The distributions $\Pmk$ need to be well behaved for the problem to be tractable.
We will assume that they satisfy concentration inequalities of
the following form. For all $\epsilon>0$,
\begin{align*}
\forall\,m,\,k,\hspace{0.3in}
\PP\big(\Xbarmks - \mumk > \epsilon\big) \;<\; \nu e^{-s\psi(\epsilon)}, \hspace{0.3in}
\PP\big(\Xbarmks - \mumk < -\epsilon\big) \;<\; \nu e^{-s\psi(\epsilon)}.
\numberthis \label{eqn:concIneq}
\end{align*}
Here $\nu>0$ and 
$\psi$ is an increasing function with $\psi(0) = 0$ and
is at least increasing linearly $\psi(x) \in \bigOmega(x)$.
For example, if the distributions are sub-Gaussian, then
$\psi(x) \in \bigTheta(x^2)$.

The performance of a multi-fidelity strategy which switches from low to high
fidelities can be worsened by artificially inserting fidelities. Consider a
scenario where $\costmm{m+1}$ is only slightly larger than $\costm$ and
$\zetamm{m+1}$ is only slightly smaller than $\zetam$. This situation is
unfavourable since there isn't much that can be inferred from the  $(m+1)$\ssth
fidelity that cannot already be inferred from the $m$\ssth by expending the same cost.
We impose the following regularity condition to avoid such situations.

\begin{assumption}
\label{asm:decayCondn}
The $\zetam$'s decay fast enough such that  
$\sum_{i=1}^m \frac{1}{\psi(\zetai)} \leq \frac{1}{\psi(\zetamm{m+1})}$
for all $m<M$.
\label{asm:decayCondn}
\end{assumption}

Assumption~\ref{asm:decayCondn} 
is not necessary to analyse our algorithm, however, the performance of
\mfucbs when compared to \ucbs is most appealing when the above holds.
In cases where $M$ is small enough and can be treated as a constant, the assumption
is not necessary.
For sub-Gaussian distributions, the condition is satisfied for
an exponentially decaying
$(\zetaone, \zetatwo,\dots)$ such as $(1/\sqrt{2}, 1/2, 1/2\sqrt{2}\dots)$.

Our goal is to design a strategy $\Acal_0$ that has low expected pseudo-regret 
$\EE[\Rc(\COST, \Acal_0)]$ for
all values of (sufficiently large) $\COST$, i.e. the equivalent of an anytime
strategy, as opposed to a fixed time horizon strategy, in the usual bandit setting.
The expectation is over the observed rewards which also dictates the number of
plays $N$.
From now on, for simplicity we will write $R(\COST)$ when $\Acal$ is
clear from context and refer to it just as regret.

\section{The Multi-Fidelity Upper Confidence Bound (\mfucb) Algorithm}
\label{sec:mfucb}

As the name suggests, the  \mfucbs algorithm maintains an upper confidence bound 
corresponding to $\mumk$ for each $m\in \{1,\dots,M\}$ and $k\in \mathcal{K}$ based 
on its previous plays. Following \ucbs 
strategies~\citep{auer03ucb,audibert09multiarmedbandit}, we 
define the following set of upper confidence bounds, 
\begin{align*}
\BBmkst \;&=\; 
\;\Xbarmks \,+\, \psiinv\Big(\frac{\rho\log t}{s}\Big) + \zetam, 
\hspace{0.2in} \textrm{for all } \;m\in \left\{ 1,\ldots,M \right\}, \;k\in \mathcal{K}\\
\BBkt \;&=\; \min_{m=1,\dots,M} \BBmkTt.
\numberthis \label{eqn:BktDefn}
\end{align*}
Here $\rho$ is a parameter in our algorithm and $\psi$ is from~\eqref{eqn:concIneq}.
Each $\BBmkTt$ provides a high probability upper bound on $\muMk$
with their minimum $\BBkt$ giving the tightest bound (See Appendix~\ref{app:upperBound}). 
Similar to \ucb, at time $t$ we play the arm $\It$ with the highest upper 
bound $\It = \argmax_{k\in\allArms} \BBkt$. 

Since our setup has multiple fidelities associated with each arm, the 
algorithm needs to determine at each time $t$ which fidelity ($\mt$) to play the 
chosen arm ($\It$). 
For this consider an arbitrary fidelity $m<M$. 
The $\zetam$ conditions on $\mumk$ imply a constraint on the value of $\muMk$.
If, at fidelity $m$, the uncertainty interval  $\psiinv(\rho\logt/\TmIttmo)$ is large, 
then we have not constrained $\muMIt$ sufficiently well yet. 
There is more information to be gleaned about $\muMIt$ 
from playing the arm $I_t$ at fidelity $m$. 
On the other hand, playing at fidelity $m$ indefinitely will
not help us much since the $\zetam$ elongation of the confidence band caps off how
much we can learn about $\muMIt$ from fidelity $m$; i.e. even if we knew $\mumIt$, we
will have only constrained $\muMIt$ to within a $\pm\zetam$ interval.
Our algorithm captures this natural intuition. Having selected $\It$, we begin checking
at the first fidelity.
If $\psiinv(\rho\logt/\ToneIttmo)$ is smaller than a
threshold $\gammaone$ we proceed to check the second fidelity, continuing in a
similar fashion. 
If at any point  $\psiinv(\rho\logt/\TmIttmo) \geq\gammam$,
we play $\It$ at fidelity $\mt = m$. If we go all the way to fidelity
$M$, we play at $\mt=M$.
The resulting procedure is summarised below in Algorithm~\ref{alg:mfucb}.

\insertAlgoMFUCB

\textbf{Choice of $\gammam$:}
In our algorithm, we choose
\begin{align*}
\gammam = \psiinv\left( \frac{\costm}{\costmm{m+1}}\psi\big(\zetam\big)\right)
\numberthis \label{eqn:gammamdefns}
\end{align*}
To motivate this choice, note that 
if $\Deltamk = \muMopt-\mumk - \zetam > 0$ then we can conclude that arm $k$ is not
optimal. 
Step 2 of the algorithm attempts to eliminate arms for which 
$\Deltamk\gtrsim\gammam$ from plays above the $m$\ssth fidelity.
If $\gammam$ is too large, then we would not eliminate a sufficient number of arms
whereas if it was too small we could end up playing a suboptimal arm $k$ 
(for which $\mumk > \muMopt$) too many times at fidelity $m$.
As will be revealed by our analysis, the given choice represents an optimal tradeoff
under the given assumptions.


\section{Analysis}
\label{sec:analysis}
\label{sec:Rcanalysis}

We will be primarily concerned with the term
$\Rctilde(\COST,\Acal) = \Rctilde(\COST)$ from~\eqref{eqn:regretDefn}. 
$\rctilde(\COST,\Acal)$ is a residual term; 
it is an artefact of the fact that after the $N+1$\superscript{th} play, the spent
capital would have exceeded $\COST$. For any algorithm that operates oblivious to a 
fixed capital, it can be
bounded by $\costM\muMopt$ which is negligible compared to $\Rctilde(\COST)$.
According to the above, we have the following expressions for $\Rtilde(\COST)$:
\begin{align*}
\hspace{-0.05in}
\Rctilde(\COST) \;=\; 
\sum_{k\in\allArms} \DeltaMk \left(\hspace{-0.01in}\sum_{m=1}^M \costm \TmkN \right),
\numberthis \label{eqn:RsRcdefns}
\end{align*}

Central to our analysis will be the following partitioning of $\allArms$. 
First denote the set of arms whose fidelity $m$ mean is within $\eta$ of $\muMopt$ to
be $\Jcalme = \{k\in\allArms;\,\muMopt - \mumk \leq \eta\}$. 
Define $\Kcalone \triangleq \Jcalcmmee{1}{\zetaone + 2\gammaone}
=\{k\in\allArms;\,\Deltaonek > 2\gammaone\}$
to be the arms whose first fidelity mean
$\muonek$ is at least $\zetaone + 2\gammaone$ below the optimum $\muMopt$. Then
we recursively define,
\vspace{-0.05in}
\begingroup
\allowdisplaybreaks
\begin{align*}
&\Kcalm \triangleq \Jcalcmzg \cap \bigg(\bigcap_{\ell=1}^{m-1}\Jcallzg \bigg), 
\,\forall  m\hspace{-0.03in}\leq \hspace{-0.03in} M-1,
\hspace{0.18in} 
\KcalM \triangleq \complement{\optArms} \cap 
\bigg(\bigcap_{\ell=1}^{M-1}\Jcallzg \bigg). 
\end{align*}
\endgroup
Observe that for all $k\in\Kcalm$, $\Deltamk > 2\gammam$ and
$\Deltalk \leq 2\gammal$ for all $\ell<m$. 
For what follows, for any $k\in\Kcal$, 
$\kmap$ will denote the partition $k$ belongs to,
i.e. $\kmap=m$ s.t. $k\in\Kcalm$.
We will see that $\Kcalm$ are the arms that will be played at the $m$\ssth
fidelity but can be excluded from fidelities higher than $m$ using information at
fidelity $m$. 
See Fig.~\ref{fig:Kcalms} for an illustration of these partitions.

\insertFigSets

\subsection{Regret Bound for \mfucbs}
\label{sec:upperbound}

Recall that $N=\sum_{m=1}^M\TmN$ is the total (random) number of plays by a
multi-fidelity strategy within capital $\COST$.
Let $\nCOST=\floor{\COST/\costM}$ be the (non-random) number of plays by any strategy
that operates only on the highest fidelity. 
Since $\costm<\costM$ for all $m<M$,
$N$ could be large for an arbitrary multi-fidelity method. However, our analysis
reveals that for \mfucb, $N\lesssim\nCOST$ with high probability.
The following theorem bounds $\Rc$ for \mfucb. The proof is given in
Appendix~\ref{app:upperBound}. For clarity, we ignore the
constants but they are fleshed out in the proofs.

\begin{theorem}[Regret Bound for \mfucb]
\label{thm:RcBound}
Let $\rho>4$. There exists $\COST_0$ depending on $\costm$'s such that for 
all $\COST>\COST_0$, \emph{\mfucbs} satisfies,
\begin{align*}
\frac{\EE[\Rc(\COST)]}{\lognl}
\;\lesssim \quad  \sum_{k\notin\optArms} \DeltaMk \cdot
    \frac{\costkmap}{\psiDeltakmapkb}
\quad\asymp\quad
  \sum_{m=1}^M \sum_{k\in\Kcalm}
  \DeltaMk \frac{\costm}{\psiDeltamkb}  
\end{align*}
\end{theorem}
Let us compare the above bound to \ucbs whose regret is 
$\frac{\EE[R(\COST)]}{\lognl}\asymp\sum_{k\notin\optArms} 
\DeltaMk \frac{\costM}{\psiDeltaMkb}$.
We will first argue that \mfucbs does not do significantly worse than \ucbs in the 
worst case.
Modulo the $\DeltaMk\lognl$ terms, regret for \mfucbs due to arm $k$ is
$\Rckmfucb \asymp \costkmap/\psiDeltakmapkb$. 
Consider any $k\in\Kcalm$, $m<M$ for which $\Deltamk > 2\gammam$. 
Since 
\[\DeltaMk \leq\; \Deltakmapk + 2\zetakmap \;\lesssim\;
\psiinv\Big(\frac{\costmm{\kmap+1}}{\costkmap}\psiDeltakmapkb\Big),
\] 
a (loose) lower
bound for \ucbs for the same quantity is 
$\Rckucb \asymp \costM/\psiDeltaMkb \gtrsim
 \frac{\costM}{\costmm{\kmap+1}}\Rckmfucb$.
Therefore for any $k\in \Kcalm, m<M$, \mfucbs is at most a constant times worse than
\ucb.
However, whenever $\Deltakmapk$ is comparable to or larger than $\DeltaMk$, \mfucbs
outperforms \ucbs by a factor of $\costkmap/\costM$ on arm $k$.
As can be inferred from the theorem, most of the cost invested by \mfucbs on arm $k$
is at the $\kmap$\ssth fidelity.
For example, in Fig.~\ref{fig:Kcalms},  \mfucbs would not play 
the yellow arms $\Kcalone$ beyond the first fidelity (more than a 
constant number of times).
Similarly all green and red arms are played mostly at the second and third
fidelities respectively.
Only the blue arms are played at
the fourth (most expensive) fidelity. On the other hand \ucbs plays all
 arms at the fourth fidelity.
Since lower fidelities are cheaper \mfucbs achieves better regret than \ucb.

It is essential to note here that $\DeltaMk$ is small for arms in in $\KcalM$.
These arms  are close to the optimum and require
more effort to distinguish than arms that are far away.
\mfucb, like \ucbs, invests $\lognl\costM/\psiDeltaMkb$ capital in those arms.
That is, the multi-fidelity setting does not help us significantly with the
``hard-to-distinguish'' arms.
That said, in cases where $K$ is very large and the sets $\KcalM$ is small the bound
for \mfucbs can be appreciably better than \ucb.

%
%

\subsection{Lower Bound}
\label{sec:lowerbound}

Since, $N\geq \nCOST = \floor{\COST/\costM}$, any multi-fidelity strategy which 
plays a suboptimal arm a polynomial number of times at any fidelity
after $n$ time steps, 
will have worse regret than \mfucbs (and \ucb).
Therefore, in our lower bound we will only consider strategies
which satisfy the following condition.

\begin{assumption}
Consider the strategy after $n$ plays at any fidelity.
For any arm with $\DeltaMk>0$, 
we have 
$\EE[\sum_{m=1}^M\Tmkn] \in \littleO(n^a)$ for any  $a>0$ \label{asm:lbAsm}.
\end{assumption}

For our lower bound 
we will consider a set of Bernoulli distributions $\Pmk$ for each fidelity $m$ 
and each arm $k$ with mean $\mumk$.
It is known that for Bernoulli distributions
$\psi(\epsilon)\in\bigTheta(\epsilon^2)$~\citep{wasserman10allOfStat}.
To state our lower bound we will further partition the set $\Kcalm$ into
two sets $\Kcalmone, \Kcalmtwo$ as follows,
\begin{align*}
\Kcalmone = \{k\in\Kcalm\;: \Deltalk \leq 0\; \forall \ell < m\},
\hspace{0.3in}
\Kcalmtwo 
= \{k\in\Kcalm\;: \exists\,\ell<m\,\; \suchthat\; \Deltalk > 0\}.
\end{align*}
For any $k\in\Kcalm$ our lower bound, given below, is different depending on
which set $k$ belongs to. 
\begin{theorem}[Lower bound for $R(\COST)$]
Consider any set of Bernoulli reward distributions with $\muMopt\in(1/2, 1)$ and
$\zetaone < 1/2$. Then, for any strategy satisfying Assumption~\ref{asm:lbAsm} the
following holds.
\begin{align}
\liminf_{\COST\rightarrow\infty} \frac{\EE[\Rc(\COST)]}{\lognl}
\;\geq\;\;\; c\cdot\, \sum_{m=1}^M \left[
\sum_{k\in\Kcalmone} \DeltaMk\frac{\costm}{\Deltamksq} 
\;+\; \sum_{k\in\Kcalmtwo} \DeltaMk\min_{\ell \in \Lcalmk} 
\frac{\costl}{\Deltalksq}
\right]
\end{align}
Here $c$ is a problem dependent constant.
$\Lcalmk = \{\ell<m: \Deltalk>0\}\cup\{m\}$ is the union of the $m$\ssth
fidelity and all fidelities smaller than $m$ for which $\Deltalk>0$.
\label{thm:Rclb}
\end{theorem}
Comparing this with Theorem~\ref{thm:RcBound} we find that \mfucbs meets the lower
bound on all arms $k\in\Kcalmone,\;\forall m$. However, it may be loose on
any  $k\in\Kcalmtwo$.
The gap can be explained as follows. For $k\in\Kcalmtwo$, there exists
some $\ell<m$ such that $0<\Deltalk<2\gammal$.
As explained previously, the switching criterion of \mfucbs ensures that we do not 
invest too much effort
trying to distinguish whether $\Deltalk <0$ since $\Deltalk$ could be very small.
That is, we proceed to the next fidelity only if we cannot conclude
$\Deltalk\lesssim \gammal$. However, since $\costm>\costl$ it might be the case that
$\costl/\Deltalksq  < \costm/\Deltamksq$ even though $\Deltamk>2\gammam$.
Consider for example a two fidelity problem where 
$\Delta = \Deltaonek=\Deltatwok< 2\sqrt{\costone/\costtwo}\zetaone$.
Here it makes sense to distinguish the arm as being suboptimal at the
first fidelity with $\costone\lognl/\Delta^2$ capital instead of 
$\costtwo\lognl/\Delta^2$ at the second fidelity. However, \mfucbs distinguishes this
arm at the higher fidelity as $\Delta<2\gammam$ and therefore does not meet the lower
bound on this arm.
While it might seem tempting to
switch based on estimates for $\Deltaonek,\Deltatwok$, this idea is not desirable
as estimating $\Deltatwok$ for an arm requires
$\lognl/\psi(\Deltatwok)$ samples at the second fidelity; this is is exactly what we
are trying to avoid for the majority of the arms via the multi-fidelity setting.
We leave it as an open problem to resolve this gap.

%
%
%


\section{Proof Sketches}
\label{sec:sketches}

\subsection{Theorem~\ref{thm:RcBound}}
\vspace{-0.02in}
First we analyse \mfucbs after $n$ plays (at any fidelity) and control the
number of plays of an arm at various fidelities depending on which $\Kcalm$ it
belongs to.
To that end we prove the following.
\begin{lemma}\textbf{\emph{(Bounding $\EE[\Tmkn]$ -- Informal)}}
\label{lem:EETmknBounds}
After $n$ time steps of \emph{\mfucbs} for any $k\in\Kcal$,
\begin{align*}
\Tlkn \lesssim \frac{\logn}{\psigammam},\;\;\forall\,\ell<\kmap,\hspace{0.4in} 
\EE[\Tkmapkn] \lesssim \frac{\logn}{\psi(\Deltakmapk/2)}, \hspace{0.4in}
\EE[\Tgkmapkn] \leq \bigO(1). 
\end{align*}
\end{lemma}
The bounds above are illustrated in Table~\ref{tb:upperbound}.
Let $\Rtildek(\COST) = \sum_{m=1}^M\costm\DeltaMk\TmkN$ be the regret incurred due 
to arm $k$ and $\Rtildekn = \EE[\Rtildek(\COST)|N=n]$.
Using Lemma~\ref{lem:EETmknBounds} we have,
\begin{align*}
\frac{\Rtildekn}{\DeltaMk\logn} \;\lesssim\;
\sum_{\ell=1}^{\kmap -1}\frac{ \costl}{\psigammam}
\;\;+\;\; \frac{\costkmap}{\psiDeltakmapk}
\;+\;\littleO(1)
\numberthis
\label{eqn:RtildenBound}
\end{align*}
The next step will be to control the number of plays $N$ within capital $\COST$ which
will bound $\EE[\logN]$.
While $\COST/\costone$ is an easy bound, we will see that for
\mfucb, $N$ will be on the order of $\nCOST =\COST/\costM$.
For this we will use the following high probability bounds on $\Tmkn$.

\insertEETmknTable

\begin{lemma}
\label{lem:PPTmknBounds}\textbf{\emph{(Bounding $\PP(\Tmkn>\cdot\;)$ -- Informal)}}
After $n$ time steps of \emph{\mfucbs} for any $k\in\Kcal$, 
\begin{align*}
\PP\left(\Tkmapkn \;\gtrsim  \;x\cdot \frac{\logn}{\psiDeltakmapk}  \right)
\;\lesssim\; \frac{1}{n^{x\rho-1}}\,,  \hspace{0.4in}
\PP\left(\Tgkmapkn >  x \right) \;\lesssim\;
  \frac{1}{x^{\rho -2}}.
\end{align*}
\end{lemma}
We bound the number of plays at fidelities less than $M$ via
Lemma~\ref{lem:PPTmknBounds} and obtain
$n/2 > \sum_{m=1}^{M-1}\Tmn$ with probability  greater than, say
$\delta$, for all $n\geq \nzero$.
By setting $\delta=1/\log(\COST/\costone)$, we get
$\EE[\log(N)] \lesssim \log(\nCOST)$.
The actual argument is somewhat delicate since 
$\delta$ depends on $\COST$.

This gives as an expression for the regret due to arm $k$
to be of the form~\eqref{eqn:RtildenBound} where $n$ is replaced by
$\nCOST$. Then we
we argue that the regret incurred by an arm $k$ at fidelities less than $\kmap$
(first term in the RHS of~\eqref{eqn:RtildenBound})
is dominated by $\costkmap/\psi(\Deltakmapk)$ (second term). 
This is possible due to the design of the sets $\Kcalm$
 and Assumption~\ref{asm:decayCondn}.
While Lemmas~\ref{lem:EETmknBounds},~\ref{lem:PPTmknBounds} require only $\rho>2$,
 we need $\rho>4$ to ensure that $\sum_{m=1}^{M-1}\Tmn$
remains sublinear when we plug-in the probabilities from
Lemma~\ref{lem:PPTmknBounds}.
$\rho>2$ is attainable with a more careful design of the sets $\Kcalm$.
The $\COST>\COSTz$ condition is needed because initially \mfucbs is playing at
lower fidelities and for small $\COST$, $N$ could be much larger than $\nCOST$.

\subsection{Theorem~\ref{thm:Rclb}}

First we show that for an arm $k$ with $\Deltapk>0$ and $\Deltalk\leq 0$
for all $\ell<p$, any strategy should satisfy
\[
R_k(\COST)
\;\gtrsim\;\; \lognl \,\DeltaMk
\bigg[\min_{\ell\geq p, \Deltalk>0}\,
\frac{\costl}{\Deltalksq} \bigg]
\]
where $R_k$ is the regret incurred due to arm $k$.
The proof uses a change of measure argument.
The modification has Bernoulli distributions with mean $\mutildelk, \ell=1,\dots,M$
where $\mutildelk = \mulk$ for all $\ell<m$. Then we push $\mutildelk$
slightly above $\muMopt-\zetal$ from $\ell=m$ all the way to $M$ where
$\mutildeMk>\muMopt$. 
To control the probabilities after changing to
$\mutildelk$ we use the conditions in Assumption~\ref{asm:lbAsm}.
Then for $k\in\Kcalm$ we argue that $\costl\Deltalksq \gtrsim \costm/\Deltamksq$
using, once again the design of the sets $\Kcalm$. This yields the 
separate results for $k\in\Kcalmone,\Kcalmtwo$.

\vspace{-0.1in}

\section{Some Simulations on Synthetic Problems}
\label{sec:simulations}
\vspace{-0.1in}

We compare \ucbs against \mfucbs on a series of synthetic problems.
The results are given in Figure~\ref{fig:simOne}.
Due to space constraints, the details on these experiments are given in
Appendix~\ref{app:simulations}.
Note that \mfucbs outperforms \ucbs on all these problems. Critically, note that the
gradient of the curve is also smaller than that for \ucbs -- corroborating our
theoretical insights.
We have also illustrated the number of plays by \mfucbs and \ucbs at each fidelity
 for one of these problems. 
The arms are arranged in increasing order of $\muMk$ values.
As predicted by our analysis, most of the very suboptimal arms are only played at
the lower fidelities.
As lower fidelities are cheaper, \mfucbs is able to use more higher fidelity plays
at arms close to the optimum than \ucb.

\vspace{-0.1in}

\insertSimOneFigure

\section{Conclusion}
\label{sec:conclusion}
\vspace{-0.1in}
We studied a novel framework for studying exploration
exploitation trade-offs when cheaper approximations to a desired experiment are
available. 
We propose an algorithm for this setting, \mfucb, based on upper confidence bound
techniques. It uses the cheap lower fidelity
plays to eliminate several bad arms and reserves the expensive high fidelity queries
for a small set of arms with high expected reward,
hence achieving better regret than strategies which ignore multi-fidelity information.
We complement this result with a lower bound which demonstrates that \mfucbs is 
near optimal. 

Other settings for bandit problems with multi-fidelity evaluations might warrant
different definitions for the regret.
For example, consider a gold mining robot where
each high fidelity play is a real world experiment of the robot and incurs cost
$\costtwo$.
However, a vastly cheaper computer simulation which incurs $\costone$
approximate a robot's real world behaviour. 
In applications like this $\costone \ll \costtwo$. However, unlike our setting lower
fidelity plays may not have any rewards (as simulations do not yield actual gold).
Similarly, in clinical trials the regret due to a bad treatment at the high fidelity,
would be, say, a dead patient. However, a bad treatment at a lower fidelity may not
warrant a large penalty. These settings are quite challenging and we wish to work on
them going forward.




\newpage

\vspace{0.2in}
\subsubsection*{Acknowledgements}
We would like to thank anonymous reviewers from COLT 2016 and NIPS 2016 for the helpful
comments on the content, structure and presentation.
This research is partly funded by DOE grant DESC0011114.

{\small
\renewcommand{\bibsection}{\section*{References\vspace{-0.1em}} }
\setlength{\bibsep}{1.1pt}
\bibliography{kky,herbert}
}
\bibliographystyle{plainnat}

\newpage
\appendix

\section{Upper Bound}
\label{app:proofs}
\label{app:upperBound}

We provide the proof in the following three subsections.
We will repeatedly use the following result.
\insertprespacing
\begin{lemma}
\label{lem:BBkopttBound}
For all $u > 0$, $\sum_{t=u+1}^\infty \PP(\BBkoptt < \muMopt) 
\leq \frac{M\nu}{\rho-2} \frac{1}{u^{\rho-2}}$.
\end{lemma}
\vspace{-0.15in}
\begin{proof}
The proof is straightforward using a union bound.
\begingroup
\allowdisplaybreaks
\begin{align*}
&\PP(\BBkoptt<\muMopt) = 
  \PP(\exists m\in \{1,\dots,M\},\; \exists\, 1\leq s \leq t-1,
   \;\;\BBmkoptst \leq \muMopt)
\numberthis\label{eqn:BBkopttBound}\\
&\hspace{0.2in}
  = \summM\sum_{s=1}^{t-1}  \PP \bigg(\Xbarmkopts - \mummkk{m}{\kopt}  <
     \muMopt - \mummkk{m}{\kopt} - \zetam -
        \psiinv\bigg(\frac{\rho\logt}{s}\bigg) \bigg) \\
&\hspace{0.2in}
  \leq\;\summM\sum_{s=1}^{t-1} \nu t^{-\rho}  \leq\; M\nu t^{1-\rho}
\end{align*}
\endgroup
In the third step we have used $\muMopt-\mumk \leq \zetam$.
The result follows by bounding the sum with the integral
$\sum_{t=u+1}^\infty t^{1-\rho} \leq \int_u^\infty t^{1-\rho} = 
u^{2-\rho}/(\rho-2)$.
\end{proof}

\subsection{Proof of Lemma~\ref{lem:EETmknBounds}}
\label{app:EETmknProof}

We first provide a formal statement of Lemma~\ref{lem:EETmknBounds}.
\begin{lemma}
\label{lem:EETmknBoundsFull}
Let $m\leq M$ and consider any arm $k\in\Kcalm$. 
After $n$ time steps of \emph{\grmfucbs} with $\rho>2$ and $\gamma>0$,
we have the following bounds on $\EE[\Tlkn]\;$ for $\ell =1,\dots,M$. 
\begin{align*}
\Tlkn \leq \frac{\rho\logn}{\psigammam} + 1,\;\;\forall\,\ell<m,\hspace{0.4in} 
\EE[\Tmkn] \leq \frac{\rho\logn}{\psi(\Deltamk/2)} + \kapparho, \hspace{0.4in}
\EE[\Tgmkn] \leq \kapparho. 
\end{align*}
Here, $\kapparho = 1 + \frac{\nu}{2} + \frac{M\nu}{\rho-2}$ is a constant.
\end{lemma}

\begin{proof}
As $n$ is fixed in this proof, we will write $\EE[\cdot], \PP(\cdot)$ for
$\EE[\cdot|N=n], \PP(\cdot|N=n)$. Let $\phimt =
\floor{\frac{\rho\logt}{\psigammam}}$. 
By design of the algorithm we won't play any arm more than
$\phimn + 1$ times at any $m<M$. To see this, assume we have already played
$\phimn+1$ times at any $t<n$. Then,
\[
\psiinv\bigg(\frac{\rho\logt}{\Tmktmo}\bigg) 
< \psiinv\left(\frac{\rho\logt}{\rho\logn}\psigammam\right) \leq \gammam,
\]
and we will proceed to the $(m+1)$\ssth fidelity in step 2 of
Algorithm~\ref{alg:mfucb}.
This gives the first part of the theorem. For any
$\ell\geq m$ we can avoid the $\frac{1}{\psigammam}$ dependence to obtain tighter 
bounds.

For the case $\ell=m$, our analysis follows usual multi-armed bandit
analyses~\citep{bubeck12regret,audibert09multiarmedbandit}.
For any $u\leq n$, we can bound $\Tmkn$ via $\Tmkn \leq u + \sum_{t=u+1}^n \Zmktu$
where $\Zmktu = \indfonearg{\mt=m \indAnd \It = k \indAnd \Tmktmo \geq u}$.
We relax $\Zmktu$ further via,
\begin{align*}
\Zmktu &\leq \indfonearg{\Tlktmo > \philt \;\forall \ell\leq m-1 \indAnds
  u\leq \Tmktmo \leq\phimt \indAnds \BBkt > \BBkoptt } \\
  &\leq \indfonearg{ \Tmktmo \geq u \indAnds \BBmkTt \geq \muMopt } \;+\;
      \indfonearg{\BBkoptt < \muMopt} \\
  &\leq \indfonearg{\exists\,u\leq s\leq t-1:\; \BBmkst>\muMopt} \;+\;
      \indfonearg{\BBkoptt < \muMopt}. 
\end{align*}
This yields, $\EE[\Tmkn] \leq u + \sum_{t=u+1}^n \sum_{s=u}^{t-1} 
\PP(\BBmkst>\muMopt) +
\sum_{t=u+1}^n\PP(\BBkoptt<\muMopt)$. 
The third term in this summation is bounded by $M\nu/(\rho-2)$ using 
Lemma~\ref{lem:BBkopttBound}.
To bound the second, choose 
$u = \ceil{\rho \logn/\psi(\Deltamk/2)}$.
Then, 
\begingroup
\allowdisplaybreaks
\begin{align*}
\PP(\BBmkst>\muMopt) 
&= \PP\bigg(\Xbarmks -\mumk > \muMopt - \mumk - \zetam - \psirhologts \bigg) \\
&\leq \PP(\Xbarmks -\mumk > \Deltamk/2)
\leq  \nu \exp\bigg(-s\psi\Big(\frac{\Deltamk}{2}\Big)\bigg) \leq \nu n^{-\rho}
\numberthis \label{eqn:PPBBmkst}
\end{align*}
\endgroup
In the second and last steps we have used
$\psiinv(\rho\logt/s) < \psiinv(\rho\logt/u) \leq \Deltamk/2$ since $\psiinv$ is
increasing and $u>\rho\logn/\psi(\Deltamk/2)$.
Since there are at most $n^2$ terms in the summation, the second term is bounded by
$\nu n^{2-\rho}/2 \leq \nu/2$.
Collecting the terms gives the bound on $\EE[\Tmkn]$.

To bound $\Tgmkn$ we write $\Tgmkn \leq u + \sum_{t=u+1}^n\Zgmktu$ where 
\begin{align*}
\Zgmktu &= \indfonearg{\mt>m \indAnds \It=k \indAnds \Tmmkktt{>m}{k}{t-1} \geq u} \\
&\leq \indfonearg{\Tlktmo>\philt\,\forall\ell\leq m  \indAnds \BBkt>\BBkoptt 
    \indAnds \Tmmkktt{>m}{k}{t-1} \geq u}  \\
&\leq \indfonearg{\Tmktmo >\phimt \indAnds \BBmkTt > \muMopt} + 
  \indfonearg{\BBkoptt <\muMopt} \\
&\leq \indfonearg{\exists\, \phimt+1\leq s \leq t-1:\; \BBmkst > \muMopt} + 
    \indfonearg{\BBkoptt < \muMopt}
\end{align*}
This yields, $\EE[\Tgmkn] \leq u + \sumuon\sum_{s=\phimt+1}^{t-1}\PP(\BBmkst > \muMopt)
\;+\; \sumuon\PP(\BBkoptt<\muMopt)$.
The inner term inside the double summation can be bounded via,
\begin{align*}
\PP(\BBmkst>\muMopt) 
&= \PP\left(\Xbarmks -\mumk > \muMopt - \mumk - \zetam - \psirhologts\right) \\
&\leq \PP(\Xbarmks -\mumk > \Deltamk - \gammam)
\leq \nu \exp(-s\psi(\Deltamk-\gammam)) \\
&\leq \nu \exp\left(-\frac{\psi(\Deltamk-\gammam)}{\psi(\gammam)} \rho\logt\right)
\leq \nu t^{-\rho}
\numberthis \label{eqn:PPBBgmkst}
\end{align*}
The second step follows from
$s>\phimt > \rho\logt/\psi(\gammam)$ and the last step uses 
$\psi(\Deltamk-\gammam) > \psi(\gammam)$ when $\Deltamk > 2\gammam$.
To bound the summation, we use $u=1$ and bound it by an integral:
$\sum_{t=u+1}^nt^{-\rho+1} \leq 1/(2u^{\rho-2}) \leq 1/2$.
Collecting the terms gives the bound on $\EE[\Tgmkn]$.
\end{proof}

\subsection{Proof of Lemma~\ref{lem:PPTmknBounds}}
\label{app:PPTmknBounds}
 
We first provide a formal statement of Lemma~\ref{lem:PPTmknBounds}.
\begin{lemma}
\label{lem:PPTmknBoundsFull}
Consider any arm $k\in\Kcalm$. For \emph{\grmfucbs} with $\rho>2$ and $\gamma>0$,
we have the following concentration results for $\ell =1,\dots,M$ for any $x\geq1$.
\begin{align*}
\PP\left(\Tmkn >  x\bigg(1+\rholognpsiDelta\bigg)  \right)
\;&\leq\; \frac{\nu\kappamkrho}{(x\cdot\logn)^{\rho-1}} + \frac{\nu}{n^{x\rho-1}}.  \\
\PP\left(\Tgmkn >  x \right) \;&\leq\;
  \frac{M\nu}{\rho-1}\frac{1}{x^{\rho-1}} + \frac{1}{(\rho-2)x^{\rho -2}}
\end{align*}
Here, $\kappamkrho = \frac{M}{\rho-1}\left(\frac{\psiDeltamk}{\rho}\right)^{\rho-1}$.
\end{lemma}

\begin{proof}
For the first inequality, we modify the analysis in~\citet{audibert09multiarmedbandit}
to the multi-fidelity setting.
We begin with the following observation for all $u\in\NN$.
\begin{align*}
&\{\forall\,t:\,u+1\leq t \leq n, \BBmkut\leq \muMopt\} \;\cap 
\numberthis \label{eqn:technicalOne} \\
&\hspace{0.2in}
\bigcap_{m=1}^M \{\forall\,1\leq s \leq n-u: \BBmkoptsups>\muMopt\} 
\;\;\implies \;\; \Tmkn \leq u
\end{align*}
To prove this, 
consider $\sm, m=1,\dots,M$ such that $\sone\geq 1$, $\sm \geq 0, \forall m\neq 1$. 
For all 
$u+\summM\sm\leq t\leq n$ and for all $\ell=1,\dots,M$ we have
\[
\BBmmkksstt{\ell}{\kopt}{\sell}{t} \geq \BBmmkksstt{\ell}{\kopt}{\sell}{u+s}
\;>\; \muMopt \geq \BBmkut \geq \BBmmkksstt{m}{k}{\Tmktmo}{t}.
\]
This means that arm $k$ will not be the $\BBkt$ maximiser at any time $u<t<n$
and consequently it won't be played more than $u+1$ times at the $m$\ssth fidelity.
Via the union bound we have,
\begin{align*}
\PP(\Tmkn>u) &\leq \sumuon\PP(\BBmkut > \muMopt) + 
  \summM\sum_{s=1}^{n-u} \PP(\BBmmkksstt{m}{\kopt}{s}{u+s} < \muMopt).
\end{align*}
We will use $u=\ceil{x(1+\rho\logn/\psiDeltamk)}$.
Bounding the inner term of the second double summation closely mimics 
the calculations in~\eqref{eqn:BBkopttBound} via which it can be shown
$\PP(\BBmmkksstt{m}{\kopt}{s}{u+s} < \muMopt) \leq \nu (u+s)^{-\rho}$.
The second term is then bounded by an integral as follows,
\begin{align*}
  \summM\sum_{s=1}^{n-u} \PP(\BBmmkksstt{m}{\kopt}{s}{u+s} < \muMopt)
  \leq M\sum_{s=1}^{n-u}\nu(u+s)^{-\rho} \leq M\nu\int_{u}^nt^{-\rho}
  \leq \frac{M\nu u^{1-\rho}}{\rho-1} 
  \leq \frac{\nu \kappamkrho}{(x\cdot \logn)^{\rho-1}}
\end{align*}
The inner term of the first summation mimics the calculations
in~\eqref{eqn:PPBBmkst}. Noting that $s>x\rho\logn/\psiDeltamk$ it can be shown
$\PP(\BBmkut>\muMopt) \leq \nu n^{-\rho x}$ which bounds the outer summation by
$\nu n^{-\rho x + 1}$. This proves the first concentration result.

For the second, we begin with the following observation for all $u\in\NN$.
\begin{align*}
&\{\forall\,t:\,u+1\leq t \leq n,\;\; \BBmkTt\leq \muMopt \indOrs \Tmktmo \leq \phimt\} 
  \;\cap
\numberthis \label{eqn:technicalTwo} \\
&\hspace{1.2in}
\bigcap_{m=1}^M \{\forall\,1\leq s \leq n-u: \BBmkoptsups>\muMopt\} 
\;\;\implies \;\; \Tgmkn \leq u
\end{align*}
To prove this first note that when $\Tmktmo \leq\phimt$ we will play at the
$m$\superscript{th} fidelity or lower. 
Otherwise, consider $\sm, m=1,\dots, M$ such that $\sone \geq 1$ and $\sm \geq 0,
\forall m$. 
For all  $u+\summM\sm\leq t\leq n$ and for all $\ell=1,\dots,M$ we have
\[
\BBmmkksstt{\ell}{\kopt}{\sell}{t} \geq \BBmmkksstt{\ell}{\kopt}{\sell}{u+s}
\;>\; \muMopt \geq \BBmmkksstt{m}{k}{\Tmktmo}{t}.
\]
This means that arm $k$ will not be played at time $t$ and consequently for any
$t>u$. After a further relaxation we get,
\begin{align*}
\PP(\Tgmkn>u) &\leq \sumuon\sum_{s=\phimt+1}^{t-1}
  \PP(\BBmkst > \muMopt) + 
  \summM\sum_{s=1}^{n-u} \PP(\BBmmkksstt{m}{\kopt}{s}{u+s} < \muMopt)
\end{align*}
The second summation is bounded via $\frac{M\nu}{(\rho-1)u^{\rho-1}}$.
Following an analysis similar to~\eqref{eqn:PPBBgmkst}, the inner term of the first
summation can be bounded by $\nu t^{-\rho}$ which bounds the first term
by $u^{2-\rho}/(\rho-2)$. 
The result follows by using $u=x$
in~\eqref{eqn:technicalTwo}.
\end{proof}

\subsection{Proof of Theorem~\ref{thm:RcBound}}
\label{app:ERCOSTBound}

We first establish the following Lemma.
\insertprespacing
\begin{lemma}[Regret of \mfucb]
\label{lem:ERCOSTBound}
Let $\rho>4$. There exists $\COST_0$ depending on $\costone,\costM$ such that for 
all $\COST>\COST_0$, \emph{\grmfucbs} satisfies,
\begingroup
\allowdisplaybreaks
\begin{align*}
&\EE[R(\COST)] \leq \muMopt \costM \,+\,
  \sum_{k=1}^K \DeltaMk\left(
    \sum_{\ell=1}^{[k]-1} \costl \frac{\rho\lognlpc}{\psigammal}
  \,+\, \costkmap \frac{\rho\lognlpc}{\psiDeltakmapk} 
  \,+ \muMopt\kapparho\costM
  \right) 
\end{align*}
\endgroup
Here $c=1 + \log(2)$ and $\kapparho = 1 + \frac{\nu}{\rho - 2} + \frac{M\nu}{\rho-2}$ 
are constants.
\end{lemma}

Denote the set of arms ``above" $\Kcalm$ by $\Kcalupm = \bigcup_{\ell=m+1}^M\Kcall$
and those ``below" $\Kcalm$ by $\Kcaldownm = \bigcup_{\ell=1}^{m-1}\Kcall$.
We first observe,
\begin{align*}
&\bigg(\forall\,m\leq M-1,\;\forall k\in\Kcalm, \;\;
\Tmkn \leq x\bigg(1 + \rholognpsiDelta\bigg) \indAnds \Tgmkn \leq y \bigg) 
\numberthis \label{eqn:TnmImplies} \\
&\hspace{0.0in}\implies
\sum_{m=1}^{M-1}\Tmn \;\leq \; Ky \;+\;
\sum_{m=1}^{M-1}\sum_{k\in\Kcalm} x\bigg(1 + \rholognpsiDelta\bigg) +
\sum_{m=1}^{M-1}|\Kcalupm| \left(1+\rholognpsigammam\right)
\end{align*}
To prove this we first note that the LHS of~\eqref{eqn:TnmImplies} is reducible to,
\begin{align*}
\forall\,m\leq M-1,\;\; 
\Tmn \leq \sum_{k\in\Kcaldownm}\Tmkn \;+\;
\sum_{k\in\Kcalm} x\bigg(1 + \rholognpsiDelta\bigg) +
\sum_{k\in\Kcalupm} \left( 1+\rholognpsigammam \right)
\end{align*}
The statement follows by summing the above from $m=1,\dots,M-1$ and rearranging the
$\Tgmkn$ terms to obtain,
\begingroup
\allowdisplaybreaks
\begin{align*}
&\sum_{m=1}^{M-1}\sum_{k\in\Kcaldownm}\Tmkn =
\sum_{m=1}^{M-1}\sum_{\ell=1}^{m-1} \sum_{k\in\Kcall}\Tmkn =
\sum_{m=1}^{M-2}\sum_{k\in\Kcalm} \sum_{\ell=m+1}^{M-1}\Tlkn 
\leq\;\sum_{m=1}^{M-2}\sum_{k\in\Kcalm} \Tgmkn  \\
&\hspace{0.2in}\;\leq\;
(K-|\Kcalmm{M-1}\cup\KcalM\cup\optArms|)y \;\leq\; Ky.
\end{align*}
\endgroup

Now for the given $\COST$ under consideration, define $\deltal =
\frac{1}{\log(\COST/\costone)}$. In addition define,
\begingroup
\allowdisplaybreaks
\begin{align*}
\xnd &= \max\left( \;\;1\;\;,\;\;
\frac{1}{\rho}\left(3+\frac{\log(2\nu\pi^2K/(3\delta))}{\logn}\right)\;\;,\;\;
\left(\frac{2\pi^2K\nu M}{3(\rho-1)\delta}\right)^{\frac{1}{\rho-1}}
  \frac{\psiDeltamk}{\rho} n^{\frac{2}{\rho-1}} 
 \right). \\
\ynd &= \max\left( \;\;1\;\;,\;\;
\left(\frac{2\pi^2 KM\nu}{3(\rho-1)\delta}\right)^{\frac{1}{\rho-1}} 
  n^{\frac{2}{\rho-1}}\;\;,\;\;
  \left(\frac{\pi^2 K}{3\delta}\right)^{\frac{1}{\rho-2}}
  n^{\frac{2}{\rho-2}}
\right).
\end{align*}
\endgroup
Now choose $\nzerol$ to be the smallest $n$ such that the following holds for all
$n\geq\nzerol$.
\begin{align}
\frac{n}{2} \;\geq \; K\yndl \;+\;
\sum_{m=1}^{M-1}\sum_{k\in\Kcalm} \xndl\bigg(1 + \rholognpsiDelta\bigg) +
\sum_{m=1}^{M-1}\sum_{k\in\Kcalupm} 1+\rholognpsigamma,
\label{eqn:n0condns}
\end{align}
For such an $\nzerol$ to exist, for a given $\COST$, 
we need both $\xn,\yn$ sublinear.
This is true since $\rho > 4$.
In addition, observe that $\nzerol$ grows only polylogarithmically in
$\COST$ since~\eqref{eqn:n0condns} reduces to 
$n^p \gtrsim (\log(\COST))^{1/2}$ where $p>0$ depends on our choice of $\rho$.

By~\eqref{eqn:TnmImplies}, the RHS of~\eqref{eqn:n0condns} is an upper bound on the
number of plays at fidelities lower than $M$. Therefore, for all $n \geq \nzerol$,
\begin{align*}
\PP\Big(\TMn<\frac{n}{2}\Big)
&\leq \sum_{m=1}^{M-1} \sum_{k\in\Kcalm}
\PP\bigg(\Tmkn > \xnd\bigg(1 + \rholognpsiDelta\bigg)\bigg) + 
\PP\bigg(\Tgmkn > \ynd \bigg) 
\numberthis \label{eqn:QnmPbound} \\
&\leq \sum_{m=1}^{M-1} \sum_{k\in\Kcalm}
\frac{\nu}{n^{\rho\xndl-1}} + \frac{\nu\kappamkrho}{(\xndl\logn)^{\rho-1}}
+ \frac{\nu M}{(\rho-1)\yndl^{\rho-1}} + \frac{1}{2\yndl^{\rho-2}} \\
&\leq K \left( 4\times \frac{3\delta}{2Kn^2\pi^2} \right) \leq \frac{6\delta}{n^2\pi^2}.
\end{align*}
The last step follows from the fact that each of the four terms inside the summation
in the second line are $\leq3\delta/(2Kn^2\pi^2)$. For the last
term we have used that $(\rho-2)/2 > 1$ and that $3\delta/(\pi^2K)$ is smaller
than $1$. Note that the double summation just enumerates over all arms
in $\Kcal$.

We can now specify the conditions on $\COSTz$.
$\COSTz$ should be large enough so that for all $\COST\geq \COSTz$, we have
$\floor{\COST/\costM} \geq \nzerol$. Such an $\COSTz$ exists since $\nzerol$ 
grows only polylogarithmically in $\COST$. This ensures that we have played
a sufficient numer of rounds to apply the concentration result
in~\eqref{eqn:QnmPbound}.

Let the (random) expended capital after $n$ rounds of \mfucbs be $\Omega(n)$. 
Let $\Ecal = \{\exists n\geq \nzerol: \Omega(n) < n\costM/2\}$. 
Since $\Omega(n) \geq \costM\TMn$, by using the union bound on~\eqref{eqn:QnmPbound}
we have $\PP(\Ecal) \leq \deltal$. Therefore,
\begin{align*}
\PP\bigg(N>\frac{2\COST}{\costM}\bigg) \;=\;
  \PP\left( N > \frac{2\COST}{\costM} \Big| \Ecal\right) 
  \underbrace{\PP(\Ecal)}_{\leq \,\deltal}
  \;+\; 
  \underbrace{\PP\left( N > \frac{2\COST}{\costM} \Big| \Ecal^c\right)}_{=\,0} 
  \PP(\Ecal^c) 
  \;<\; \deltal
\end{align*}
The last step uses the following reasoning: 
Conditioned on $\Ecal^c$, $n> 2\Omega(n)/\costM$ is false for $n>\nzerol$.
In particular, it is true for the \emph{random} number of plays $N$ since
$\COST>\COSTz \implies N\geq \nzerol$.
Now, clearly $\COST > \Omega(N)$ and therefore $N>2\COST/\costM$ is also false.

%
%

By noting that $\nCOST = \COST/\costM$ and that $\log(\COST/\costone)$ is always an
upper bound on $\logN$, we have,
\begin{align*}
\EE[\log(N)] \leq \log(2\nCOST) \PP(N<2\nCOST) +
\log\left(\frac{\COST}{\costone}\right)  \PP(N>2\nCOST)
\;\leq\; \log(\nCOST) + 1 + \log(2)
\numberthis \label{eqn:logNBound}
\end{align*}

Lemma~\ref{lem:ERCOSTBound} now follows by an application of
Lemma~\ref{lem:EETmknBounds}. First we condition on $N=n$ to obtain,
\begingroup
\begin{align*}
&\EE[R(\COST)|N=n] \;\leq\; \muMopt\costM \;\;+\;\; 
\sum_{k=1}^K\sum_{m=1}^M\DeltaMk\costm\Tmkn \\
& \hspace{0.2in} \leq\;
  \muMopt\costM \,+\;
  \sum_{k=1}^K \DeltaMk\left(
    \sum_{\ell=1}^{[k]-1} \costl \frac{\rho\logn}{\psigammam}
  \;\;+\;  \costkmap \frac{\rho\logn}{\psiDeltakmapk} 
  \;+ \kapparho\costM
  \right) 
\end{align*}
\toworkon{Gautam can you look into this? Reviewer one had said the additive constants
$1$ and $\kappa_k$ are missing.}
\endgroup
The theorem follows by plugging in the above in $\EE[R(\COST)] =
\EE[\EE[R(\COST)|N]]$ and using the bound for $\EE[\logN]$ in~\eqref{eqn:logNBound}.


We can now bound the regret for \mfucb.

\begin{proof}[\textbf{Proof of Theorem~\ref{thm:RcBound}}]
Recall that $\psigammam = \frac{\costm}{\costmm{m+1}}\psi(\zetam)$.
Plugging this into Lemma~\ref{lem:ERCOSTBound} 
we get
\begin{align*}
\EE[\Rc(\COST)] &\leq\; \muMopt \costM \;+\;
  \sum_{k=1}^K \DeltaMk\left(
    \sum_{\ell=1}^{[k]-1}\costmm{\ell+1} \frac{\rho\lognlpc}{\psizetal}
  \;\;+\; \costkmap \frac{\rho\lognlpc}{\psiDeltakmapk} 
  \;+\kapparho\costM
  \right) \\
  &\leq\; 
\muMopt \costM \;+\;
  \sum_{k=1}^K \DeltaMk\cdot\costkmap\rho\lognlpc 
  \left( \frac{2}{\psizetakmapmo}
  + \frac{1}{\psiDeltakmapk} \right)
  + \DeltaMk\kapparho\costM
\end{align*}
The second step uses Assumption~\ref{asm:decayCondn}.
The theorem follows by noting that for any $k\in\Kcalm$ and $\ell<m$,
$\Deltamk  = \Deltalk + \zetal - \zetam + \mulk-\muMk + \muMk-\mumk
\leq 2\gammal + 2\zetal \leq 4\zetal$.
Therefore $1/\psiDeltamkb > c_1/\psi(\zetal)$ where $c_1$ depends on $\psi$ (for
sub-Gaussian distributions, $c_1=1/16$).
\end{proof}

\section{Lower Bound}
\label{app:lb}

The regret $R_k$ incurred by any multi-fidelity strategy
after capital $\COST$ due to a suboptimal arm $k$ is,
\[
R_k(\COST) = \DeltaMk\sum_{m=1}^M \costm \TmkN,
\]
here $N$ is the total number of plays.
We then have, $R(\COST) = \sum_k R_k(\COST)$.
For what follows, for an arm $k$ and any fidelity $m$ denote
$\klmk = \kl{\mumk}{\muMopt-\zetam}$.
The following lemma provides an asymptotic lower bound on $R_k$.
\insertprespacing
\begin{lemma}
Consider any set of Bernoulli reward distributions with $\muMopt \in(1/2, 1)$ and
$\zetaone < 1/2$.
For any $k$ with $\Deltalk < 0$ for all $\ell<p$ and $\Deltapk >0$,
there exists a problem dependent constant $c_p$ such that
any strategy satisfying Assumption~\ref{asm:lbAsm} 
must satisfy, 
\[
\liminf_{\COST\rightarrow \infty} \frac{R_k(\COST)}{\lognl}
\;\geq\; c'_p \,\DeltaMk\min_{\ell\geq p, \Deltalk>0}\, 
\frac{\costl}{\Deltalksq} 
\]
\label{lem:lbLemma}
\end{lemma}
\insertpostspacing
\begin{proof}
For now we will fix $N=n$ and consider any game after $n$ rounds.
Our proof we will modify the reward distributions of the given arm $k$ for all 
$\ell\geq p$ and show
that any algorithm satisfying Assumption~\ref{asm:lbAsm} will not be able to
distinguish between both problems with high probability.
Since the KL divergence is continuous, for any $\epsilon>0$ we can choose 
$\mutildepk \in(\muMopt - \zetap, \muMopt - \zetap + \min_{\ell<p}-\Deltalk)$ 
such that $\kl{\mupk}{\mutildepk} < (1+\epsilon)\kl{\mupk}{\muMopt-\zetap}
=(1+\epsilon)\klpk$.

The modified construction for arm $k$, will also have Bernoulli distributions with means
$\mutildeonek,\mutildetwok,\dots,\mutildeMk$. $\mutildepk$ will be picked
to satisfy the two constraints above and for the remaining fidelities,
\[
\mutildelk = \mulk \hspace{0.1in}\textrm{for } \ell < p, \hspace{0.4in}
\mutildelk = \mutildepk + \zetap - \zetal \hspace{0.1in}
\textrm{for } \ell > p.
\]
Now note that for $\ell<p$, $\mutildeMk - \mutildelk = \mutildepk + \zetap - \mulk
< \muMopt -\zetap - \Deltalk + \zetap - \mulk =\zetal$;
similarly, $\mutildeMk - \mutildelk = \mutildepk + \zetap - \mulk
> \muMopt - \mulk > \muMk - \mulk > - \zetal$.
For $\ell > p$, $\mutildeMk - \mutildelk = \zetal$. Hence, the modified construction
satisfies the conditions on the lower fidelities laid out in Section~\ref{sec:prelims}
and we can use Assumption~\ref{asm:lbAsm}.
Further $\mutildeMk > \muMopt$, so $k$ is the optimal arm in the modified problem.
Now we use a change of measure argument.

Following~\citet{lai85bandits,bubeck12regret},
denote the expectations, probabilities and distribution in the original problem as 
$\EE,\PP, P$ and in the modified problem as $\EEtilde,\PPtilde, \tilde{P}$.
Denote a sequence of observations when playing arm $k$ at 
by $\{\Zlkt\}_{t\geq 0}$ and define,
\[
\klhatlks \;=\; \sum_{t=1}^s 
  \log \left(
  \frac{ \mulk \Zlkt + (1-\mulk)(1-\Zlkt) }{ \mutildelk \Zlkt +
(1-\mutildelk)(1-\Zlkt)}
  \right)
  \;=
  \sum_{t:\Zlkt = 1}\hspace{-0.05in} \log \frac{\mulk}{\mutildelk} \;+
  \sum_{t:\Zlkt = 0}\hspace{-0.05in} \log \frac{1-\mulk}{1-\mutildelk}.
\]
Observe that $\EE[s^{-1} \klhatlks] = \kl{\mulk}{\mutildelk}$. 
Let $A$ be any event in the $\sigma$-field generated by the observations in the game.
\begingroup
\allowdisplaybreaks
\begin{align*}
\PPtilde(A) &= \int\indfone(A)\ud\tilde{P} 
  = \int \indfone(A) \prod_{\ell=p}^M \,
\Bigg( \prod_{i=1}^{\Tlkn} \frac{\Ptildelk(\Zlki)}{\Plk(\Zlki)} \Bigg) \ud P \\
  &= \EE\bigg[\indfone(A) \exp\bigg(-\sum_{\ell\geq p}\klhatlkT\bigg) 
\bigg] \numberthis \label{eqn:tilting}
\end{align*}
\endgroup
Now let $\fln = C\logn$ for all $\ell$ such that $\Deltalk<0$ and 
$\fln = \frac{1}{M-p}\frac{1-\epsilon}{\kl{\mupk}{\mutildepk}}\logn$ otherwise.
(Recall that $\Deltalk<0$ for all $\ell<p$ and $\Deltapk>0$).
$C$ is a large enough constant that we will specificy shortly.
Define the following event $\An$.
\[
\An = \left\{ \Tlkn \leq \fln, \;\forall \ell  \indAnds
\klhatlkT \leq \frac{1}{M-p}(1-\epsilon/2)\logn, \;\forall \ell: \Deltalk>0
\right\}
\]
By~\eqref{eqn:tilting} we have $\PPtilde(\An) \geq \PP(\An) n^{-(1-\epsilon/2)}$.
Since $k$ is the unique optimal arm in the modified construction, 
by Assumptions~\ref{asm:lbAsm} we have $\forall\, a>0$,
\begin{align*}
\PPtilde(\An) \leq \PP\Big(\sum_m\Tmkn < \bigTheta(\logn)\Big) \leq 
\frac{\EE\big[n-\sum_m\Tmkn\big]}{n-\bigTheta(\logn)} \in \littleO(n^{a-1}) \\
\end{align*}
\toworkon{Reviewer 1 had said that the $\PP,\EE$ should have tilde's. I don't think
so. Can you check?}
By choosing $a<\epsilon/2$ we have $\PP(\An)\rightarrow 0$ as $n\rightarrow \infty$.
Next, we upper bound the probability of $\An$ in the original problem as follows,
\[
\PP(\An) \geq \PP\Bigg( \underbrace{\Tlkn \leq \fln,\;\forall \ell}_{\Anone}
  \indAnds
  \underbrace{\max_{s\leq\fn}\klhatlks \leq 
    \frac{1}{M-p}(1-\epsilon/2) \log n, \forall \ell: \Deltalk>0
}_{\Antwo} \Bigg)
\]
We will now show that $\Antwo$ remains
large as $n\rightarrow 0$. 
Writing $\Antwo = \bigcap_{\ell:\Deltalk>0} \Antwol$, we have
\[
\PP(\Antwol) = \PP\left( \frac{\fln(M-p)}{(1-\epsilon)\logn}\cdot \frac{1}{\fln}
\max_{s\leq \fln} \klhatmks \leq \frac{1-\epsilon/2}{1-\epsilon} \right).
\]
As $\fln\rightarrow\infty$, by the strong law of large numbers
$\frac{1}{\fln}\max_{s\leq \fln} \klhatmks \rightarrow \kl{\mumk}{\mutildemk}$. After
substituting for $\fln$ and repeating for all $\ell$, 
we get $\lim_{n\rightarrow\infty} \PP(\Antwo) = 1$.
Therefore, $\PP(\Anone) \leq \littleO(1)$. 
To conclude the proof, we upper bound $\EE[R_k(\COST)]$ as follows,
\begin{align*}
\frac{\EE[R_k(\COST)]}{\DeltaMk} \;&\geq\; 
  \PP( \exists\,\ell\, \suchthat \TlkN > \flN  )\cdot
  \EE[R_k(\COST)\,|\, \exists\,\ell\, \suchthat \TlkN > \flN ] 
  \;\geq\; \PP(\complement{\Anone})\cdot\min_\ell \flnl \costl \\
  &\geq\;(1-\littleO(1)) \min_{\ell\geq p} \frac{(1-\epsilon)\lognl \costl}
    {(M-p)\kl{\mulk}{\mutildelk}} 
  \;\;\geq\; \frac{\lognl}{M-p} (1-\littleO(1)) 
  \frac{1-\epsilon}{1+\epsilon}
  \min_{\ell>m} \frac{\costl}{\kllk}
\end{align*}
Above, the second step uses the fact that $N\geq \nCOST$ and $\log$ is increasing.
In the third step, we have chosen $C>\max_{\ell\geq p} \costl
\DeltaMk/\kl{\mulk}{\mutildelk}$ for $\ell<p$ large enough so that the minimiser will
be at $\ell\geq p$. The lemma follows by noting that the statements holds for all
$\epsilon>0$ and that for Bernoulli distributions
with parameters $\mu_1,\mu_2$, $\kl{\mu_1}{\mu_2} \leq
(\mu_1-\mu_2)^2/(\mu_2(1-\mu_2))$.
The constant given in the theorem is $c'_p =
\frac{1}{M-p}\min_{\ell>p}(\muMopt-\zetal)(1-\muMopt+\zetal)$.
\end{proof}

%

We can now use the above Lemma to prove theorem~\ref{thm:Rclb}.

\begin{proof}[\textbf{Proof of Theorem~\ref{thm:Rclb}}]
Let $k\in\Kcalmone$. We will use Lemma~\ref{lem:lbLemma} with $p=m$. 
It is sufficient to show that $\costl/\Deltalksq \gtrsim \costm/\Deltamksq$ for all
$\ell>m$.
First note that 
\[
\Deltalk = \muMopt-\mumk -\zetam + \mumk - \muMopt +\muMopt-\mulk + \zetam-\zetal
\leq \Deltamk + 2\zetam
\leq 2\Deltamk\sqrt{\frac{\costmm{m+1}}{\costm}}
\]
Here the last step uses that $\Deltamk >2\gammam = \sqrt{\costm/\costmm{m+1}}\zetam$.
Here we have used $\psi(\epsilon) = 2\epsilon^2$ which is just Hoeffding's
inequality.
Therefore, $\frac{\costm}{\Deltamksq} \leq 4\frac{\costmm{m+1}}{\Deltalksq}
\leq 4\frac{\costmm{\ell}}{\Deltalksq}$.

When $k\in\Kcalmtwo$, we use Lemma~\ref{lem:lbLemma} with 
$p=\ellz = \min\{\ell;\Deltalk > 0\}$.
However, by repeating the same argument as above, we can eliminate all $\ell>m$.
Hence, we only need to consider $\ell$ such that $\ellz\leq \ell \leq m$ and
$\Deltalk>0$ in the minimisation of Lemma~\ref{lem:lbLemma}. This is precisely the
set $\Lcalmk$ given in the theorem.
The theorem follows by repeating the above argument for all arms $k\in\Kcal$.
The constant $c_p$ in Theorem~\ref{thm:Rclb} is $c'_p/4$ where $c'_p$ is from
Lemma~\ref{lem:lbLemma}.
\end{proof}

%

\section{Details on the Simulations}
\label{app:simulations}

We present the details on the simulations used in the experiment.
Denote $\vec{\zeta} = (\zetaone,\zetatwo,\dots,\zetaM)$
and $\vec{\lambda} = (\costone,\costtwo,\dots,\costM)$.
Figure~\ref{fig:simArms} illustrates the mean values of these arms.

\begin{enumerate}
\item Gaussian: $M=500$, $M=3$, $\vec{\zeta} = (0.2, 0.1, 0)$,
  $\vec{\lambda} = (1,10,1000)$. \\
  The high fidelity means were chosen to be a uniform grid in $(0,1)$.
  The Gaussian distributions had standard deviation $0.2$.

\item Gaussian: $M=500$, $M=4$, $\vec{\zeta} = (1, 0.5, 0.2, 0)$,
  $\vec{\lambda} = (1, 5, 20, 50)$. \\
  The high fidelity means were sampled from a $\Ncal(0,1)$ distribution.
  The Gaussian distributions had standard deviation $1$.

\item Bernoulli: $M=200$, $M=2$, $\vec{\zeta} = (0.2, 0)$,
  $\vec{\lambda} = (1, 10)$. \\
  The high fidelity means were chosen to be a uniform grid in $(0.1, 0.9)$.
  The Gaussian distributions had standard deviation $1$.

\item Bernoulli: $M=1000$, $M=5$, $\vec{\zeta} = (0.5, 0.2, 0.1, 0.05, 0)$,
  $\vec{\lambda} = (1, 3, 10, 30, 100)$. \\
  The high fidelity means were chosen to be a uniform grid in $(0.1, 0.9)$.
  The Gaussian distributions had standard deviation $1$.

\end{enumerate}

In all cases above, the lower fidelity means were sampled uniformly within 
a $\pm\zetam$ band around $\muMk$.
In addition, for the Gaussian distributions we modified the lower fidelity means of
the optimal arm $\mumkk{\kopt}, m<M$ to be lower than the corresponding mean of a
suboptimal arm.
For the Bernoulli rewards, if $\mumk$ fell outside of $(0,1)$ its value was
truncated.
Figure~\ref{fig:simArms} illustrates the mean values of these arms.

For both \mfucbs and \ucbs we used $\rho=2$~\cite{bubeck12regret}.

\insertArmsFigure

\end{document}